\newif\ifabstract
\abstracttrue
 \abstractfalse 
\newif\iffull
\ifabstract \fullfalse \else \fulltrue \fi

\documentclass[11pt]{article}
\usepackage{comment}
\usepackage{amsfonts}
\usepackage{amssymb,bm}
\usepackage{amstext}
\usepackage{wrapfig}
\usepackage{booktabs} 
\usepackage{amsmath}
\usepackage{xspace}
\usepackage{theorem}
\usepackage[ruled,vlined]{algorithm2e}
\usepackage{graphicx}
\usepackage{url}
\usepackage{graphics}
\usepackage{colordvi}
\usepackage{colordvi}
\usepackage{subfigure}
\usepackage{xr-hyper}
\usepackage{hyperref}
\usepackage{xr-hyper}
\usepackage{hyperref}       
\usepackage{url}            
\usepackage{booktabs}       
\usepackage{amsfonts}       
\usepackage{nicefrac}       
\usepackage{microtype}      
\usepackage{microtype}
\usepackage{graphicx}
\usepackage{subfigure}
\usepackage{booktabs} 
\usepackage{amsfonts}
\usepackage{amssymb,bm}
\usepackage{amstext, enumitem}
\usepackage{amsmath}
\usepackage{xspace}
\usepackage{theorem}
\usepackage{graphicx,color}
\usepackage{url}
\usepackage{graphics}
\usepackage{colordvi, wrapfig}
\usepackage{colordvi}
\usepackage{subfigure}
\usepackage{amssymb}
\usepackage{pifont}

\advance \oddsidemargin by -0.8in
\newcommand{\myparskip}{3pt}
\parskip \myparskip
\newcommand{\sgn}{\text{sgn}}
\newcommand{\loss}{{L}_{n}(\bm{\bm{\theta}}^{*};p)}
\newcommand{\lossc}{{L}_{n}(\bm{\bm{\theta}}^{*};\bm{\lambda})}

\newcommand{\e}{\varepsilon}

\newif\ifsubmit

  \submitfalse

\ifsubmit
\newcommand{\shiyu}[1]{}
\newcommand{\ruoyu}[1]{}
\else
\newcommand{\shiyu}[1]{{\color{blue}{[Shiyu: #1]}}}
\newcommand{\ruoyu}[1]{{\color{red}{[Ruoyu: #1]}}}
\fi

\newtheorem{theorem}{Theorem}
\newtheorem{lemma}{Lemma}

\newtheorem{corollary}{Corollary}
\newtheorem{claim}{Claim}

\newtheorem{assumption}{Assumption}

\newtheorem{problem}{Problem}
\newenvironment{proof}{\par \smallskip{\bf Proof:}}{\hfill\stopproof}
\def\stopproof{\square}
\def\square{\vbox{\hrule height.2pt\hbox{\vrule width.2pt height5pt \kern5pt
\vrule width.2pt} \hrule height.2pt}}
\newcommand{\footremember}[2]{%
   \footnote{#2}
    \newcounter{#1}
    \setcounter{#1}{\value{footnote}}%
}
\newcommand{\footrecall}[1]{%
    \footnotemark[\value{#1}]%
}

\begin{document}

\title{Achieving Small Test Error in Mildly Overparameterized Neural Networks}
\author{Anonymous Authors}
\author{Shiyu Liang\footremember{uiuc}{University of Illinois at Urbana-Champaign}\\ sliang26@illinois.edu
          \and Ruoyu Sun\footrecall{uiuc}\\ ruoyus@illinois.edu
            \and R. Srikant\footrecall{uiuc}\\rsrikant@illinois.edu
           }

\date{}

\maketitle
\begin{abstract}
Recent theoretical works on over-parameterized neural nets
have focused on two aspects: optimization and generalization.
 Many existing works that study optimization and generalization
 together are based on neural tangent kernel and require
  a very large width.
In this work, we are interested in the following question:
 for a binary classification problem with two-layer mildly over-parameterized ReLU network, can we find a point with small test error in polynomial time?
 We first show that the landscape of loss functions with explicit regularization has the following property: all local minima and certain other points which are only stationary in certain directions achieve small test error. 
We then prove that for convolutional neural nets, there is an algorithm 
 which finds one of these points in polynomial time (in the input dimension and the number of data points). In addition, we prove that for a fully connected neural net, with an additional assumption on the data distribution, there is a polynomial time algorithm. 
 
\end{abstract}

\section{Introduction}

Machine learning practitioners are ultimately interested in test error. Nevertheless, due to the difficulty of directly analyzing test error, theoreticians often decompose the test error into two parts: training error and generalization gap. In the past few years, a lot of research has been devoted to a single aspect of the problem, either analyzing
 the training loss or analyzing the generalization gap.
Using optimization theory, a number of works 
\cite{kawaguchi2016deep,freeman2016topology,choromanska2015loss,nguyen2017loss,nguyen2018loss,liang2018understanding, liang2018adding, liang2019revisiting}
have analyzed local minima in deep neural networks,
or prove converge to global minima of the training loss,
but do not provide bounds on generalization error. 
Using statistical learning theory, \cite{bartlett2017spectrally,neyshabur2018towards}
analyzed the generalization gap, but did not
 show how to achieve small training error. 
There are some recent works that intend to prove small test error for neural nets. 
Most works explore the NTK-type analysis
 to prove a generalization bound (e.g.,
 \cite{arora2019fine,cao2019generalization}), but these works require a very large width, say, at least $\Omega(n^2)$. References
 \cite{ji2019, chen2019much} showed that the requirement for the width can be small if the data satisfies certain assumptions 
 (i.e., the margin is large enough), but the width can still be large for other  data distributions. 
One notable exception is
 \cite{allen2018learning}, which used a different initialization (scaling by $1/m$ instead of $1/\sqrt{m}$) and thus allow the trajectory to go beyond the 
 kernel regime.
 However, their width is still at least $\Omega(n^r)$ (this is not explicitly stated; see Section \ref{subsec: related work} for detailed discussion). 


  
   In this paper, we are interested in generic data distributions and mildly overparamterized neural nets with width in the order of $n$, where $n$ is the number of samples. The question we explore is the following: is there a polynomial time
  algorithm to find a point with small test error,
  for a mildly overparameterized neural nets? 

We add an explicit $\ell_2$ regularization term to the loss function and first study the landscape of the resulting regularized empirical loss function. 
We show that all local minima of this loss function, and other directional local minima (i.e., those points which are local minima only in certain specific directions), have good memorization and generalization performance.
In the course of identifying these properties, we also show that all of these points have a key common feature: at least one neuron is inactive at these local stationary points if the number of neurons is larger than the number of points in the dataset.
We exploit this property to design a
variant of gradient descent algorithm which finds a good set of parameters for the neural network. Specifically, we show that when the algorithm gets stuck at a directional local minimum which does not have good generalization performance, by perturbing the parameters of the inactive neuron, one can find a descent direction to reduce the value of the loss function. 
We prove that for convolutional neural net,
the algorithm runs in polynomial time in the number of data points and the input dimension.
For a fully connected neural network, under an additional
assumption on the data distribution, we provide
another polynomial time algorithm that finds a point with
small test error.

\subsection{Relationship to prior work}\label{subsec: related work}

We now relate our work to prior work in a few different categories:
\begin{itemize}

 \item \emph{Neural Tangent Kernel and Related Results:}  
 By exploring the connection between gradient descent on versions of the neural network linearized around the initialization and kernel regression on a kernel called the Neural Tangent Kernel (NTK), 
 researchers proved the convergence of gradient descent to global minima
 for ultra-wide networks
 \cite{jacot2018neural,du18,arora2019fine}.
 The width we require in this work is in the order of $O(n)$,
 which is much smaller than the requirement of the NTK-type works. 
 
 \item \emph{Optimization and Generalization}.
As mentioned earlier, a number of NTK-type results
provide small test error for neural nets.
Following the convention in NTK theory, most works (e.g.
\cite{arora2019fine,cao2019generalization})
require a very large width, say, at least $\Omega(n^2)$.
Recently, 
 \cite{ji2019}\cite{chen2019much}
 showed that the width can be smaller than $O(n)$
 if the dataset is separable
 with a large enough margin.
More specifically, the number of parameters required is inversely proportional to the margin by which the binary labelled dataset can be separated,
thus for certain dataset the width can be as small
as $\text{poly}(\log (n))$. 
However, if the dataset contains points that are very close to each other with different labels or if the dataset contains very small amount of mislabeling errors, then the number of parameters required could be very large or infinity.
One of our interests is in obtaining results for datasets which are not separable due to the fact, in practice, that a small amount of labeling errors can render a dataset non-separable.


A notable different work is 
\cite{allen2018learning}, which did not utilize NTK-type analysis and proved small test error by running gradient descent.
However, their required width is still large (at least
$\Omega(n^4)$ or $\Omega(1/\epsilon^8)$ where $\epsilon$
is the desired test error).
Note that their theorems did not explicitly state the dependence of the width $m$ on $n$ or $1/\epsilon$,
and we derive the dependence as follows.
In their proof of Lemma B.4 in page 40, 
the second and third equation after equation (B.4)
together enforce an requirement on $m$:
based on the discussions there, to let the first term
in the second equation be smaller than 
$O(\epsilon)$, we need 
$ \frac{ \sqrt{m} }{ \epsilon^2 } \frac{1}{\epsilon m} < \epsilon $,
 i.e., $m > \frac{1}{ \epsilon^8 }$. 
The relation between $N$ and $\epsilon$
can be found is discussed
in their Remark B.7: they 
are aware of a proof that can show $N$
 scales as $1/\epsilon^2$~\footnote{ Based on their Appendix B.5, technically their proof requires $ n > 1/ \epsilon^4$, thus the test error can only scale as $1/n^{1/4}$ which is worse than the typical test error $1/n^{1/2}$.}.
By using this desired scaling, their width $m$ would be at least $ \Omega( n^4)$. 
Another difference is that their data are generated by a target network, while our data are generic.



\item \emph{Escaping saddle points}.
 There are some works which  
 shed light on how easy or difficult it is to escape saddle points and reach a local or global minimum \cite{laurent2017multilinear,tian2017analytical,soltanolkotabi2019theoretical,mei2018mean,brutzkus2017globally,zhong2017recovery,li2017convergence,brutzkus2017sgd,wang2018learning,du2018power,oymak2019towards,janzamin2015beating, mondelli2018connection}.
Our work also handle certain points that 
a variant of GD may get stuck at, and we provide a problem-dependent algorithm to escape such points.
Note that these points may not be the saddle point
of the original loss function, since we consider
a special variant of GD that only operates
on a subset of parameters. 

    \item \emph{Explicit Regularization}: In \cite{wei2019regularization}, the authors show that explicit regularization can decrease the number of samples required to learn a distribution. In \cite{wei2020implicit}, the dual role of dropout as both an implicit and explicit regularization is discussed. In \cite{liang2019revisiting},
    an regularizer is added to eliminate the decreasing path to infinity for ReQU (rectified quadratic unit) network. These works did not
    directly provide a bound on the test error.
    
    
    \item \emph{Computational Complexity Results:} There has been a long line of work pointing out the difficulty of training neural networks, we refer the readers to small sample of such works in \cite{goel2020superpolynomial,vu1998infeasibility,blum1989training}. We do not directly address this issue in our paper, as our polynomial-time algorithm works only for two cases: convolutional neural networks where each neuron only takes as input a subset of the input vector to the neural network,
    or fully connected neural networks under extra
     assumption on the data. 
    
    \item \emph{Implicit Regularization}.
    One approach to understanding the performance of overparameterized networks is to show
    that the iterates generated by gradient descent exhibit some implicit bias. 
   For instance, for linear regression, gradient descent
   converges to a minimum norm solution for overparamaterized  models; for matrix factorization,
   gradient descent converges to a low-rank solution
   for overparameterized models \cite{gunasekar2018implicit,LiEtAl2017Algorithmic}.
 Researchers have also studied implicit
 bias in deep neural networks.
 \cite{arora2019implicit,ji2018gradient}
 analyzed deep linear networks;
 for instance, \cite{ji2018gradient} 
 showed that the weight matrices will converge
 to low-rank matrices for scalar output case.
    A few recent works
    \cite{ji2020directional,lyu2019gradient}
  analyzed deep non-linear homogeneous networks,
  but they assume the loss to be smaller than a certain threshold. They did not show that an algorithm can find a point below
  the required threshold, thus their results can be viewed
  as local analysis (near the  global minima).
 
  

\end{itemize}

\section{Preliminaries}\label{sec::prelim}

\textbf{Notation for single layered network with ReLU activation.} In this paper, we consider both fully-connected neural network (FNN) and convolutional neural network (CNN) with rectified linear units (ReLUs). For simplicity of notation, we use the following way the express the outputs of both FNNs and CNNs. Given an input vector of dimension $d$, let $\{\phi_k\}_{k\in \mathbb{N}}$ be a fixed vector series where each vector $\phi_k$ is a binary vector of dimension $d$, i.e., $\phi_k\in\{0,1\}^d$. 
We use $\odot$ to denote the Hadamard product of two vectors. Specifically, for a vector $\phi\in\{0, 1\}^d$ and an vector $x\in\mathbb{R}^d$, the Hadamard product  $x\odot \phi $ is a vector of dimension $d$ where the $k$-th coordinates of the vector $x\odot \phi$ is the product of the $k$-th coordicates of vectors $x$ and $\phi$, i.e., 
$$(x\odot \phi )(k)=x(k)\phi(k)= \left\{\begin{matrix}x(k) &\text{if }\phi(k)=1,\\ 0&\text{if }\phi(k)=0.\end{matrix}\right.$$
Given a series $\{\phi_k\}_{k\in \mathbb{N}}$, we define the output of neural network $f$ as 
$$f(x;\bm{\theta})=\sum_{j=1}^ma_j(\bm{w}_j^\top (x\odot \phi_j))_+,$$
where $m$ denotes the number of neurons in the neural network and the vector $\bm{\theta}$ consists of all parameters (i.e., $a_j$s and $\bm{w}_j$s) in the neural network. Next, we will show how to choose the series $\{\phi_k\}$ to express the outputs of FNN and CNN respectively. 

\textbf{FNN.} To represent the output of a fully-connected neural network, we can choose $\phi_k=\bm{1}_d$ for all $k\in\mathbb{N}$. This indicates $x\odot \phi_k=x$ for all $x\in\mathbb{R}^d$ and $k\in\mathbb{N}$ and thus indicates that the output of $f$ becomes 
$$f(x;\bm{\theta})=\sum_{j=1}^ma_j(\bm{w}_j^\top (x\odot \phi_j))_+=\sum_{j=1}^ma_j(\bm{w}_j^\top x)_+,$$
which is exactly the same as the notation of the output of a fully-connected neural network. We also note here that, in this case, the series $\{\phi_k\}_k$ is a periodic series with a period of one.

\textbf{CNN.}
To represent the output of a convolutional neural network with filters of size $r\in[1, d]$, we can choose a periodic series $\{\phi_k\}_{k\in\mathbb{N}}$ satisfying that for each integer $k\in[1, d-r+1]$, 
$$\phi_k(j)=\left\{\begin{matrix}1 &\text{if } k\le j\le k+r-1,\\ 0&\text{otherwise,}\end{matrix}\right.$$
 and that the periodic vector series $\{\phi_k\}_{k\in\mathbb{N}}$ has a period of $d-r+1$, i.e., $\phi_k=\phi_{k+d-r+1}$, for any $k\in\mathbb{N}$. Now it is straight-forward to see that the output of $f$ denotes a convolutional neural network with filters of size $r$.

\textbf{Loss and error.} We use $\mathcal{D}=\{(x_{i},y_{i})\}_{i=1}^{n}$ to denote a dataset containing $n$ samples, where $x_{i}\in\mathbb{R}^{d}$ and $y_{i}\in\{-1,1\}$ denote the feature vector and the label of the $i$-th sample, respectively. Given a neural network $f(x;\bm{\theta})$ parameterized by $\bm{\theta}$ and a univariate loss function $\ell:\mathbb{R}\rightarrow\mathbb{R}$, in binary classification tasks, we define the regularized  empirical loss $L_{n}(\bm{\theta};\bm{\lambda})$ as a linear combination of a regularizer $V(\bm{\theta};\bm{\lambda})$ parameterized by a vector $\bm{\lambda}$ and the average loss of the network $f$ on a sample in the dataset. We define the training error (also called the misclassification error) $R_{n}(\bm{\theta};f)$ as the misclassification rate of the network $f$ on the dataset $D$, i.e., 
\begin{equation}
L_{n}(\bm{\theta};\bm{\lambda})=\sum_{i=1}^{n}\ell(-y_{i}f(x_{i};\bm{\theta}))+V(\bm{\theta};\bm{\lambda})
\end{equation}
and
\begin{equation} R_{n}(\bm{\theta};f)=\frac{1}{n}\sum_{i=1}^{n}\mathbb{I}\{y_{i}\neq \sgn(f(x_{i};\bm{\theta}))\},
\end{equation}
where $\mathbb{I}$ is the indicator function. Given an underlying distribution $\mathbb{P}_{\bm{X}\times Y}$, we define the test error $R(\bm{\theta}; f)$ of a neural network as the misclassification error of the neural network on the underlying distribution, i.e., 
$$R(\bm{\theta}; f)=\mathbb{P}_{\bm{X}\times Y}(Y\neq \sgn(f(X;\bm{\theta}))).$$

\section{Assumptions}\label{sec::main-results}

In this section, we introduce several assumptions on the univariate loss function and dataset.

\begin{assumption}[Loss function]\label{assumption::loss-1}
 Assume that the univariate loss function $\ell$ is convex, non-decreasing and twice differentiable. Assume that both function $\ell$ and its derivative $\ell$ are $1$-Lipschitiz and that there exists a positive real number $a\in\mathbb{R}^+$ such that $\ell'(z)\le e^{az}$ holds for any $z\in\mathbb{R}$.  
\end{assumption}
\textbf{Remark: } The Lipschitz constant need not be $1,$ we assume it for the simplicity of notation. In fact, if a general univariate loss $\ell$ and its derivative $\ell'$ is $L$-Lipschitz, then we can normalize the loss function by setting $\ell_{new}=\ell/L$ to satisfy the assumption. Similarly, if $\ell'(z)\le be^{az}$ for some positive $a, b\in\mathbb{R}$ and all $z\in\mathbb{R}$, then we can also normalize the loss by setting $\ell_{new}=\ell/b$ to satisfy the assumption. 

Given a periodic series of vectors $\bm{\phi}=(\phi_k)_{k\ge 1}$, we define the class of functions $\mathcal{H}_{\bm{\phi}}$  as follows,
\begin{equation}\label{H def}
\mathcal{H}_{\bm{\phi}}=\bigcup_{m=1}^{\infty}\left\{h:x\mapsto \sum_{j=1}^{m}a_j(\bm{w}_j^\top (x\odot \phi_j))_+\Bigg|\sum_{j=1}^m|a_j|=1, \|\bm{w}_j\|_2=1,j=1,...,m\right\}.
\end{equation}
Now we introduce the assumption on the dataset. For the problem where we train the neural network to memorize all points in the dataset, we assume that the most of samples in the dataset can be separated by a neural network in the function class $\mathcal{H}_{\bm{\phi}}$ with a positive margin. 

\begin{assumption}[Dataset]\label{assumption::dataset-1}
Assume that $\|x_i\|_2\le 1$ holds for all $i\in[n]$. For a given series of vectors $\bm{\phi}=(\phi_k)_{k\ge 1}$, assume that there exists a  number $E\in[0, n]$, a margin $\gamma\in(0,1]$ and a ReLU network $h\in\mathcal{H}_{\bm{\phi}}$ such that 
$$\sum_{i=1}^n\mathbb{I}\{y_i h(x_i)\ge \gamma\}\ge n-E.$$ 
\end{assumption}

However, for the problem where we train the neural network to achieve good performance on the underlying distribution, we assume that with high probability, data samples drawn from the underlying distribution can be separated by a neural network in the function class $\mathcal{H}_{\bm{\phi}}$ with  a positive margin. 

\begin{assumption}[Data Distribution]\label{assumption::data-distribution}
Assume $\mathbb{P}_{X\times Y}(\|X\|_2\le 1)=1$. 
For a given series of vectors $\bm{\phi}=(\phi_k)_{k\ge 1}$, assume that there exists a real number $\e\in [0,0.5)$, a margin $\gamma>0$ and a ReLU network $h\in\mathcal{H}_{\bm{\phi}}$ such that $\mathbb{P}_{X\times Y}(Yh(X)\ge \gamma)=1-\e $. Assume that the samples in the dataset are independently drawn from the underlying distribution $\mathbb{P}_{X\times Y}$. 

\end{assumption}

\textbf{Remark:} In fact, the above assumption covers the several cases shown below
\begin{itemize}[leftmargin=*]
    \item \textbf{Clearly separable cases.} If all data points can be separated by a neural network $h\in\mathcal{H}_{\bm{\phi}}$ with a large margin $\gamma>0$, then $\sum_{i=1}^n\mathbb{I}\{y_i h(x_i)\ge \gamma\}\ge n$ or $\mathbb{P}_{X\times Y}(Yh(X)\ge \gamma)=1$.
    \item \textbf{Almost clearly separable cases.} If most of samples drawn from the underlying distribution can be separated by a neural network $h\in\mathcal{H}$ with a large margin $\gamma>0$. 
\end{itemize}

\textbf{Remark:} 
Assumption~\ref{assumption::dataset-1} 
is different from Assumption \ref{assumption::data-distribution}
as the latter requires independent samples
while the former does not. 
We distinguish Assumption~\ref{assumption::dataset-1} and
Assumption~\ref{assumption::data-distribution} because 
some of our results require 
Assumption ~\ref{assumption::dataset-1} 
while some of our results require 
Assumption \ref{assumption::data-distribution}. 
Further, it is straightforward to show that under Assumption~\ref{assumption::data-distribution}, when the number of samples is sufficiently large, Assumption~\ref{assumption::dataset-1} holds with high probability.

\section{All Local Minima Memorize and Generalize Well}\label{sec::single}
In this section, we will show that all local minima of the empirical loss have good memorization and generalization performance. Recall that for a single-layered ReLU network consisting of $m$ neurons, the output of the neural network is defined as $f(x;\bm{\theta})=\sum_{j=1}^{m}a_{j}\left(\bm{w}_{j}^{\top}(x\odot \phi_j)\right)_+$. Now we define the empirical loss as 

\begin{equation}\label{eq::loss-single}
L_{n}(\bm{\theta};\bm{\lambda})=\sum_{i=1}^{n}\ell(-y_{i}f(x_{i};\bm{\theta}))+\frac{1}{2}\sum_{j=1}^{m}\lambda_{j}\left[a_{j}^{2}+\|\bm{w}_{j}\|^{2}_{2}\right],
\end{equation}
where all regularizer coefficients $\lambda_{j}$'s are positive numbers and the vector $\bm{\lambda}=(\lambda_{1},...,\lambda_{m})$ consists of all regularizer coefficients. We note that after adding the regularizer, the empirical loss $L_n$ is coercive (i.e., $L_n(\bm{\theta})\rightarrow\infty$ as $\|\bm{\theta}\|_2\rightarrow\infty$) and always has a  global minimum. Now we present the following theorem to show that if the network size is larger than the dataset size ($m\ge n+1$ for FNN; for CNN, the threshold is $(n+1)(d-r+1)$)
and the regularizer coefficient vector $\bm{\lambda}$ is carefully chosen, then every local minimum of the empirical loss $\lossc$ achieves zero training error on the dataset $\mathcal{D}$.

\begin{theorem}\label{thm::single}
Let $m\ge (n+1)(d-r+1)$ and $\lambda_0\in(0, n)$. Under Assumption~\ref{assumption::loss-1} and \ref{assumption::dataset-1}, there exists  a zero measure set $\mathcal{C}\subset\mathbb{R}^{m}$ such that for any $\bm{\lambda}\in(\lambda_0/2,\lambda_{0})^{m}\setminus \mathcal{C}$, both of the following statements are true:
\begin{itemize}[leftmargin=*]
\item[(1)] the empirical loss $L_{n}(\bm{\theta};\bm{\lambda})$ is coercive.
\item[(2)] every local minimum $\bm{\theta}^{*}$ of the loss $\lossc$ achieves a training error at most $\frac{\lambda_0+2E}{ \ell'(0)\gamma n}$, i.e., $R_{n}(\bm{\theta^{*}};f)\le\frac{\lambda_0+2E}{ \ell'(0)\gamma n}$.
\end{itemize}
\end{theorem}
\textbf{Remarks:} The proof is provided in Appendix~\ref{appendix::proof::memo}. (1) The first part shows that the empirical loss is coercive and thus eliminate the possibility that a descent algorithm can diverge to infinity. (2) If we are choosing $\lambda_0=\sqrt{n}$ and if the portion of samples that cannot be separated with a large margin is small (i.e., $E/n$ is small), then the training error at every local minimum is of order $\mathcal{O}\left(\frac{1}{\gamma \sqrt{n}}+\frac{E}{\gamma n}\right)$ and thus very small when the number of samples $n$ is very large. (3) For the fully connected neural network where $r=d$, the amount of neurons we need is just $n+1$. 

One may wonder whether over-parameterization
leads to overfitting. 
The next theorem states that the test
error of every local minimum is bounded above. 

\begin{theorem}\label{thm::single-gen}
Let $m\ge (n+1)(d-r+1)$ and $\lambda_0\in(0, n)$. Assume that the data samples in the dataset are independently drawn from an underlying distribution $\mathbb{P}_{X\times Y}$. Under Assumption~\ref{assumption::loss-1} and \ref{assumption::dataset-1}, with probability at least $1-\delta$, there exists  a zero measure set $\mathcal{C}\subset\mathbb{R}^{m}$ such that for any $\bm{\lambda}\in(\lambda_0/2,\lambda_{0})^{m}\setminus \mathcal{C}$, both of the following statements are true:
\begin{itemize}[leftmargin=*]
\item[(1)] the empirical loss $L_{n}(\bm{\theta};\bm{\lambda})$ is coercive.
\item[(2)] every local minimum $\bm{\theta}^{*}$ of the loss $\lossc$ achieves a test error at most 
\begin{align*}
\mathbb{P}(Yf(X;\bm{\theta}^*)<0)=\mathcal{O}\left(\frac{\lambda_0}{\gamma n}+\frac{E\ln n}{\gamma\lambda_0\sqrt{n}}+\frac{E}{\gamma n}+\frac{\ln n}{\gamma\sqrt{n}}+\sqrt{\frac{\log(1/\delta)}{n}}\right).
\end{align*}
\end{itemize}
\end{theorem}
The proof is provided in Appendix~\ref{appendix::proof::gen}. (1) It is easy to see from the first two terms in the upper bound that the regularizer coefficient $\lambda_0$ cannot be too small or too large since the upper bound goes to infinity when  $\lambda_0\rightarrow 0$ or  $\lambda_0\rightarrow \infty$. In fact,  if we choose $\lambda_0=\sqrt{n}\ln n$, then the upper bound becomes $\mathcal{O}\left(\frac{\ln n}{\gamma \sqrt{n}}+\frac{E}{\gamma {n}}+\sqrt{\frac{\log(1/\delta)}{n}}\right)$ and the upper bound goes to $\mathcal{O}(E/(\gamma n))$ as $n$ goes to infinity. When two classes of samples can be well-separated, i.e., $E=0$, the upper bound goes to zero as $n$ goes to infinity. (2) For the fully connected neural network where $r=d$, the amount of neurons we need is just $n+1$. 

Remark: A major difference with
earlier works on neural nets is that we directly 
bound the generalization error
of the local minimum,
while those works provided bounds on the generalization gap (the gap between training error and test error)
and did not provide bounds on the training error. 

\section{Can Other Points that Memorize and Generalize Well?}
\label{sec::other}
In the previous section, we show that every local minimum of the empirical loss memorizes and generalizes well. However, some existing work has shown that finding a local minimum can be hard. Therefore, in this section, we will show that a set of points including but not limited to local minima can also have good memorization and generalization performance. In the next section, we will show that there exists an algorithm finding a point in this set with polynomial number of computations. 

Recall that for a single-layered ReLU network consisting of $m$ neurons, the output of the neural network is defined as $f(x;\bm{\theta})=\sum_{j=1}^{m}a_{j}\left(\bm{w}_{j}^{\top}(x\odot \phi_j)\right)_+$ and we define the empirical loss as 
\begin{equation}
L_{n}(\bm{\theta};\bm{\lambda})=\sum_{i=1}^{n}\ell(-y_{i}f(x_{i};\bm{\theta}))+\frac{1}{2}\sum_{j=1}^{m}\lambda_{j}\left[a_{j}^{2}+\|\bm{w}_{j}\|^{2}_{2}\right].
\end{equation}

\begin{theorem}[Memorization]\label{thm::other-memo}
Let $m\ge (n+1)(d-r+1)$. Under Assumption~\ref{assumption::loss-1} and \ref{assumption::dataset-1},  if there exists a number $C>0$ such that for any $k=1,...,d-r+1$,
$$\max_{\bm{u}\in\mathbb{B}^d}\left|\sum_{i=1}^n\ell'(-y_i f(x_i;\bm{\theta}^*))y_i(\bm{u}^\top (x_i\odot \phi_k))_+\right|\le C,$$
then the neural network with parameters $\bm{\theta}^*$ achieves a training error at most $\frac{C+2E}{ \ell'(0)\gamma n}$. 
\end{theorem}
\textbf{Remark:} The proof is provided in Appendix~\ref{appendix::proof::thm::other-memo}. (1) For every local minimum, we can show that  the inequality holds with $C=\lambda_0$. Thus Theorem \ref{thm::other-memo}
is a generalization of Theorem \ref{thm::single}.
(2) This Theorem shows that for any set of parameters satisfying the above inequalities with a small constant $C$, then this set of parameters can achieve small training error. 

Next, we will present the theorem showing that some points other than local minima can also achieve good generalization performance. 

\begin{theorem}[Generalization]\label{thm::other-gen}
Let $m\ge (n+1)(d-r+1)$ and choose $\bm{\lambda}\in[\lambda_0/2, \lambda_0]^m$ for some $\lambda_0>0$. Under Assumptions~\ref{assumption::loss-1} and \ref{assumption::dataset-1}, with probablity at least $1-\delta$, for any set of parameters $\bm{\theta}^*$ satisfying all of the following conditions,  
\begin{itemize}
\item[(1)] $\|\bm{w}_j^*\|_2= |a_j^*| $ holds for any $j=1,...,m$,
\item[(2)] $\|\nabla_{\bm{a}}L_n(\bm{\theta}^*;\bm{\lambda})\|_2\le \e$,
\item[(3)] there exists a positive number $C>0$ such that  for any $k=1,...,d-r+1$
$$\max_{\bm{u}\in\mathbb{B}^d}\left|\sum_{i=1}^n\ell'(-y_i f(x_i;\bm{\theta}^*))y_i(\bm{u}^\top (x_i\odot \phi_k))_+\right|\le C,$$
\end{itemize}
the testing error of the neural network $f(\cdot;\bm{\theta}^*)$ has an upper bound of 
$$ \mathbb{P}(Yf(X;\bm{\theta}^*)<0)=\mathcal{O}\left(\frac{C+E}{\gamma {n}}+\frac{(C+E)\ln n}{\gamma \lambda_0\sqrt{n}}+\frac{\e \|\bm{a}^*\|_2}{\lambda_0\sqrt{n}}+\sqrt{\frac{\log(1/\delta)}{n}}\right)
.$$
\end{theorem}
\textbf{Remarks: }The proof is provided in Appendix~\ref{appendix::proof::thm::other-gen}. (1) At every local minimum, $\varepsilon=0$ by the first condition of the critical point and $C=\lambda_0$ by the definition of the local minimum. (2) This theorem shows that if a set parameter is very closed to the critical point (i.e., $\e$ is small) and the inequalities in the third conditions holds with $C>0$, then by setting $\lambda_0=\sqrt{n}\ln n$, we can still achieve an
upper bound of an order of $\mathcal{O}\left(\frac{C+E}{\gamma {n}}+\frac{\ln n}{\gamma \sqrt{n}}+\sqrt{\frac{\log(1/\delta)}{n}}\right)$. 

An algorithm to find a good set of parameters which generalizes well is presented in Algorithm~\ref{algo::main}, which further calls two other algorithms Algorithm \ref{algo::coeff} and \ref{algo::perturb}. In the next section, we present the intuition behind the algorithms and prove that they converge in polynomial time to a good solution.

\begin{algorithm}[t]\label{algo::main}
\SetAlgoLined
\KwResult{ Parameter vector $\bm{\theta}^*=(a_1^*,...,a_m^*,a^*_1\bm{u}^*_1,...,a_m^*\bm{u}_m^*)$ }
\tcp{Initialization}
Choose $|a_j(0)|\le 1$ and $\|\bm{u}_j(0)\|_2=1$ for each $j\in[m]$ and $\bm{\lambda}(0)=\lambda_0\bm{1}_m$ and $\lambda_0\ge \sqrt{n}$ \;
Choose $C\ge \lambda_0$ and $K=\max\{\lceil L_n(\bm{\theta}(0);\bm{\lambda}(0)), 2n\rceil\}$\;
 \tcp{Setting Up The Loop Index }
 $k = 0$\;
 \tcp{Loop}
 \While{ $k\le L_n(\bm{\theta}(0);\bm{\lambda}(0))$ }{
 Call Algorithm~\ref{algo::coeff} to update regularizer coefficient $\bm{\lambda}(k)$\;
  \tcp{ Running The Gradient Descent Algorithm}
  Initialize $\bm{\Theta}(0)=(\alpha_j(0), \alpha_j(0)\bm{u}_j(k))_j$ such that $\alpha_j( 0)=a_j(k)$ for each $j\in[m]$\; 
  Setting $t=0$; $\eta_k=\frac{1}{72\max[L_n(\bm{\Theta}(0)), 2n]}$\;
  \While{ $\|\nabla_{\bm{\alpha}}L_{m}(\bm{\Theta}(t);\bm{\lambda}(k))\|_{2} >\frac{\lambda_0}{16K\sqrt{n}}$}
  {$\bm{\alpha}(t+1)=\bm{\alpha}(t)-\eta_k\nabla_{\bm{\alpha}}L_n(\bm{\Theta}( t);\bm{\lambda}(k))$, $t\leftarrow t+1$
  }
  \tcp{Gradient descent algorithm terminates at the $T_k$-th step}
  Call Algorithm~\ref{algo::perturb} to perturb the inactive neuron from $\bm{\Theta}(T_k)$ to $\bm{\theta}(k+1)$ and check the termination condition\;
  $k\leftarrow k+1$
  
 }
 \caption{Gradient Descent Algorithm}
\end{algorithm}

\begin{algorithm}[t]\label{algo::coeff}
\SetAlgoLined
{
\tcp{Update The Regularizer Coefficients}
  Choose $\bm{\lambda}(k)$ such that   for each $s\in[d-r+1]$
$$\min_{p_1,...,p_n\in\mathbb{R}}\sum_{j:\phi_j=\phi_s}\left(\lambda_j(k) -\sum_{i=1}^np_iy_i\sgn(a_j(k))(\sgn(a_j(k))\bm{u}_j^\top(k) (x_i\odot \phi_s))_+\right)^2\ge \left(\frac{\lambda_0}{8K}\right)^2,$$
$\bm{\lambda}(k)\preccurlyeq\bm{\lambda}(k-1)$ and $\|\bm{\lambda}(k)-\bm{\lambda}(k-1)\|_\infty\le \frac{\lambda_0}{2K}$\;
  
 }
 \caption{Updating Regularizer Coefficients}
\end{algorithm}

\begin{algorithm}[t]\label{algo::perturb}
\SetAlgoLined
{\KwResult{ Parameter vector $\bm{\theta}(k+1)$. }
\tcp{Perturb the inactive neuron}
   For each $s\in[d-r+1]$, finding the unit vector $\bm{v}_s^*$ such that 
  $$\left|\sum_{i=1}^n \ell_i'(T_k)y_i ({\bm{v}_s^*}^\top (x_i\odot \phi_s))_+\right|\ge \max_{\bm{u}\in\mathbb{B}^d}\left|\sum_{i=1}^n \ell_i'(T_k)y_i (\bm{u}^\top (x_i\odot \phi_s))_+\right|-C,$$
  where $\ell_i'(T_k)\triangleq\ell'(-y_if(x_i;\bm{\Theta}(T_k)))$\;
  \tcp{Check termination conditions}
  \eIf{ $\left|\sum_{i=1}^n \ell_i'(T_k)y_i ({\bm{v}_s^*}^\top (x_i\odot \phi_s))_+\right|\le 5\lambda_0,\forall s\in[d-r+1]$}{
  Break the loop and output $\bm{\theta}^*=\bm{\Theta}(T_{k})$\;
   }{ 
   Finding an index $s\in[d-r+1]$ such that $\left|\sum_{i=1}^n \ell_i'(T_k)y_i ({\bm{v}_s^*}^\top (x_i\odot \phi_s))_+\right|> 5\lambda_0$\;
   Finding an index $j\in[m]$ such that $\phi_j=\phi_s$ and $|a_j|\le \frac{1}{\sqrt{4n}}$\;
   $a_j(k+1)= \sgn\left(\sum_{i=1}^n \ell_i'(T_k)y_i ({\bm{v}_s^*}^\top (x_i\odot \phi_s))_+\right)\sqrt{\frac{1}{\lambda_0}}$\;
   $\bm{u}_j(k+1)=\sgn(a_j(k+1))\bm{v}_s^*$\;
   $a_r(k+1)= a_r(k), \bm{u}_r(k+1)=\bm{u}_r(k)$ for $r\neq s$\;
  }
 }
 \caption{Perturb Inactive Neurons And Check Termination Condition}
\end{algorithm}

\section{A Polynomial Time Training Algorithm}

\subsection{Algorithm}\label{sec::intuition}
In this subsection, we provide intuition as to why Algorithm~\ref{algo::main} finds parameters satisfying all conditions in Theorem~\ref{thm::other-gen}.  . 

To satisfy the first condition in Theorem~\ref{thm::other-gen} (i.e., $\|\bm{w}_j^*\|_2= |a_j^*| $ for any $j=1,...,m$), we just need to set $\bm{w}_j= \alpha_j\bm{u}_j, a_j=\alpha_j$ for some unit vector $\bm{u}_j$ and update the scalar $\alpha_j$ and the unit vector $\bm{u}_j$, respectively. Now the parameter vector  becomes $\bm{\Theta}=(\alpha_j, \alpha_j\bm{u}_j)_{j=1,\dots, m}$. Considering the update rule for each $\alpha_j$, i.e., 
\begin{align*}
\alpha_j(t+1)&=\alpha_j(t)-\eta\nabla_{\alpha_j}L_n(\bm{\Theta}(t))\\
&=\left[1-2\eta\lambda_j+2\eta\sum_{i=1}^n\ell'_i(-y_i f(x_i;\bm{\Theta}(t)))y_iy_i\sgn(\alpha( t))(\sgn(\alpha(t))\bm{u}_j^\top (x_i\odot \phi_j))_+\right]\alpha_j(t)
\end{align*}
and choosing a sufficiently small step size $\eta\le \frac{1}{8n}$, we observe that $\sgn(\alpha_j(t+1))=\sgn(\alpha_j(t))=\sgn(\alpha_j(0))$ holds for all $t=0,1,2...$. This indicates that (1) we can run the gradient descent algorithm  on the vector $\bm{\alpha}$ for a sufficiently long time to ensure that the gradient norm $\|\nabla_{\bm{\alpha}}L_n(\bm{\Theta}(t);\lambda)\|_2$ is sufficiently small and that  (2) the sign of each $\alpha_j$ does not change with respect to the time. Furthermore, using the fact that $\sgn(\alpha_j(t))=\sgn(\alpha_j(0))$ for all $t$, we obtain that for each unique template feature vector $\phi_k$, $1\le k\le d-r+1$ (the series $\{\phi_k\}_{k\ge 1}$ is a periodic series of a period equal to $d-r+1$), we have 
\begin{align*}
\|\nabla_{\bm{\alpha}}L_n(\bm{\Theta}(t))\|^2_2
&=\sum_{j=1}^m\left(\lambda_j -\sum_{i=1}^n\ell'(1-y_if(x_i;\bm{\Theta}(t)))y_i\sgn(\alpha_j(0))(\sgn(\alpha_j(0))\bm{u}_j^\top (x_i\odot \phi_j))_+\right)^2\alpha_j^2(t)\\
&\ge \sum_{j:\phi_j=\phi_k}\left(\lambda_j -\sum_{i=1}^n\ell'(1-y_if(x_i;\bm{\Theta}(t)))y_i\sgn(\alpha_j(0))(\sgn(\alpha_j(0))\bm{u}_j^\top (x_i\odot \phi_k))_+\right)^2\alpha_j^2(t).
\end{align*}
Thus, if for each $k\in[d-r+1]$, we can choose regularizer coefficients $\lambda_j$s satisfying
\begin{equation}\label{eq::intuition::1}
\min_{p_1,...,p_n\in\mathbb{R}}\sum_{j:\phi_j=\phi_k}\left(\lambda_j -\sum_{i=1}^np_iy_i\sgn(\alpha_j(0))(\sgn(\alpha_j(0))\bm{u}_j^\top (x_i\odot \phi_k))_+\right)^2\ge \text{const}^2
\end{equation}
then when the gradient descent norm is small, for each unique template feature vector $\phi_k$, $k\in[d-r+1]$, the parameters of one of the neurons is also small, i.e., 
$$\|\nabla_{\bm{\alpha}}L_n(\bm{\Theta}(t))\|_2\le \delta\times\text{const}\implies \max_{k\in[d-r+1]}\min_{j:\phi_j=\phi_k} |\alpha_j|=\max_{k\in[d-r+1]}\min_{j:\phi_j=\phi_k}\|\alpha_j\bm{u}_j\|_2\le\delta.$$


Therefore, we only need to handle the third condition in Theorem~\ref{thm::other-gen}. From the prior analysis, we know that after running the gradient descent algorithm for a long time, for each template feature vector $\phi_k$, there exists a neuron with small parameters. This indicates that  we can perturb the inactive neuron in the following way to make the empirical loss decrease further. Since
 the loss function $\ell$ is gradient Lipschitz, we obtain the following inequality
$$L_n(\tilde{\bm{\theta}};\bm{\lambda})-L_n(\bm{\Theta};\bm{\lambda})\le -\sum_{i=1}^n \ell'(-y_i f(x_i;\bm{\theta}))y_i \tilde{\alpha}_j(\tilde{\alpha}_j\tilde{\bm{u}}_j^\top (x_i\odot \phi_j))_++\lambda_j\tilde{\alpha}_j^2+2n\tilde{\alpha}_j^4+n\alpha_j^2+2n\alpha_j^4,$$
when we only perturb the $j$-th neuron from $(\alpha_j, \alpha_j\bm{u}_j)$ to $(\tilde{\alpha}_j,\tilde{\alpha}_j\tilde{\bm{u}}_j)$ for an arbitrary unit vector $\tilde{\bm{u}}_j$. Since we already know that after running the gradient descent algorithm for a sufficiently long time, for each unique feature vector $\phi_k$, there exists a neuron of index $k_j$ and with $|\alpha_{k_j}|\le \delta= 1/\sqrt{4n}$, then using the fact that $\lambda_0\ge \sqrt{n}$,  we have  
$$L_n(\tilde{\bm{\theta}};\bm{\lambda})-L_n(\bm{\theta};\bm{\lambda})\le -\frac{1}{\lambda_0}\sum_{i=1}^n \ell'(-y_i f(x_i;\bm{\theta}))y_i \sgn(\tilde{\alpha}_{k_j})(\sgn(\tilde{\alpha}_{k_j})\tilde{\bm{u}}_{k_j}^\top (x_i\odot \phi_k))_++4,$$ 
when we only perturb the $k_j$-th neuron from $(\alpha_{k_j}, \alpha_{k_j}\bm{u}_{k_j})$ to $(\tilde{\alpha}_{k_j},\tilde{\alpha}_{k_j}\tilde{\bm{u}}_{k_j})$ for an arbitrary unit vector $\tilde{\bm{u}}_{k_j}$ and $|\tilde{\alpha}_{k_j}|=1/\sqrt{\lambda_0}$. Therefore, for any arbitrary unit vector $\bm{v}_k$, setting $\sgn(\tilde{\alpha}_{k_j})=\sgn\left(\sum_{i=1}^n \ell'(-y_i f(x_i;\bm{\theta}))y_i (\bm{v}_k^\top (x_i\odot \phi_k))_+\right)$ and $\tilde{\bm{u}}_{k_j}=\sgn(\tilde{\alpha}_{k_j})\bm{v}_k$, we have 
$$L_n(\tilde{\bm{\theta}};\bm{\lambda})-L_n(\bm{\theta};\bm{\lambda})\le -\frac{1}{\lambda_0}\left|\sum_{i=1}^n \ell'(-y_i f(x_i;\bm{\theta}))y_i (\bm{v}_k^\top (x_i\odot \phi_k))_+\right|+4.$$
Assume that we have an algorithm which always finds a vector $\bm{v}_k^*$ for each $k\in[d-r+1] $satisfying $$\left|\sum_{i=1}^n \ell'(-y_i f(x_i;\bm{\theta}))y_i ({\bm{v}_k^*}^\top (x_i\odot \phi_k))_+\right|\ge \max_{\bm{u}\in\mathbb{B}^d}\left|\sum_{i=1}^n \ell'(-y_i f(x_i;\bm{\theta}))y_i (\bm{u}^\top (x_i\odot \phi_k))_+\right|-C$$
for some positive constant $C>0$. 
If for each $k\in[d-r+1]$, $\bm{v}_k^*$ satisfies
$$\left|\sum_{i=1}^n \ell'(-y_i f(x_i;\bm{\theta}))y_i ({\bm{v}_k^*}^\top (x_i\odot \phi_k))_+\right|\le 5\lambda_0$$
then we terminate the algorithm; otherwise, we set the new $\alpha_{k_j}$ and $\bm{u}_{k_j}$ to $\tilde{\alpha}_{k_j}$ and $\tilde{\bm{u}}_{k_j}$, respectively, and rerun the gradient descent.
Therefore, this means that before the algorithm terminates, each time when we perturb the inactive neuron, the empirical loss  decreases by a constant
$$L_n(\tilde{\bm{\theta}};\bm{\lambda})-L_n(\bm{\theta};\bm{\lambda})\le -1.$$
It further indicates that when the algorithm terminates, we should have 
$$ \max_{\bm{u}\in\mathbb{B}^d}\left|\sum_{i=1}^n \ell'(-y_i f(x_i;\bm{\theta}))y_i (\bm{u}^\top (x_i\odot \phi_1))_+\right|\le 5{\lambda_0}+C$$
which makes the third condition in Theorem~\ref{thm::other-gen} hold. We note here that if the algorithm continues, then in the next iteration, we  have to update the regularizer coefficient vector $\bm{\lambda}$ to make the inequality~\eqref{eq::intuition::1} hold since we update the unit vector $\bm{u}_k$s. Furthermore, when we update the coefficient vector $\bm{\lambda}$, we also need to ensure the empirical loss does not increase and the coefficient vector $\bm{\lambda}$ always locates in the region $[\lambda_0/2, \lambda_0]^m$.

\subsection{Performance and Complexity Analysis}
In this subsection, we present the result on the performance and computational complexity of Algorithm~\ref{algo::main}. We first show that Algorithm~\ref{algo::main} finds the output $\bm{\theta}^*$ with small testing error if the number of samples is sufficiently large. 

\begin{theorem}\label{thm::algo}
Let $m\ge (n+1)(d-r+1)$ and $\lambda_0=\sqrt{n}\ln n$. Assume that the data samples in the dataset are independently drawn from an underlying distribution $\mathbb{P}_{X\times Y}$. Under Assumption~\ref{assumption::loss-1} and \ref{assumption::dataset-1}, with probability at least $1-\delta$, the output $\bm{\theta}^*$ of the Algorithm~\ref{algo::main} satisfies
\begin{align*}
\mathbb{P}(Yf(X;\bm{\theta}^*)<0)=\mathcal{O}\left(\frac{C+E}{\gamma n}+\frac{\ln n}{\gamma \sqrt{n}}+\sqrt{\frac{\log(1/\delta)}{n}}\right).
\end{align*}
\end{theorem}
\textbf{Remark:} The proof is provided in Appendix~\ref{appendix::proof::thm-algo}. This upper bound on the testing error depends on a constant $C$, which denotes how well we find a unit vector $\bm{v}_s^*$  for each $s\in[d-r+1]$ such that 
  $$\left|\sum_{i=1}^n \ell'(-y_if(x_i;\bm{\theta}))y_i ({\bm{v}_s^*}^\top (x_i\odot \phi_s))_+\right|\ge \max_{\bm{u}\in\mathbb{B}^d}\left|\sum_{i=1}^n \ell'(-y_if(x_i;\bm{\theta}))y_i (\bm{u}^\top (x_i\odot \phi_s))_+\right|-C.$$
If we can perfectly solve the optimization problem on the right hand side of the inequality, then $C=0$. Therefore, it is straightforward to see that the complexity and testing performance of Algorithm~\ref{algo::main} depends on the method we use to solve the above optimization problem.

For each $s\in[d-r+1]$, an easy way to solve the optimization problem 
$$\max_{\bm{u}\in\mathbb{B}^d}G(\bm{u};\phi_s,\bm{\theta})\triangleq\left|\sum_{i=1}^n \ell'(-y_if(x_i;\bm{\theta}))y_i (\bm{u}^\top (x_i\odot \phi_s))_+\right|$$
is  exhaustive search. Specifically, for each $s\in[d-r+1]$, we can search every possible vector in the following discrete set,
$$\mathcal{A}_s=\left\{\bm{u}\in\mathbb{B}^d:\bm{u}(k)=0\text{ for }k\notin I_s\text{ and }\bm{u}(k)\in\bigcup_{h=0}^{\lceil \lambda_0\rceil}\left\{\frac{h\lambda_0}{n}\right\}\text{ for } k\in I_s\right\},$$
where the index set $I_s\triangleq\{j\in[d]: \phi_s(d)=1\}$, and set the output $\bm{v}_s^*$ be the vector $\bm{u}$ maximizing $G$. Since by assumption, $|I_s|\le r$ for  each $s\in[d-r+1]$, then this exhaustive search algorithm has a total computation complexity of $\mathcal{O}((d-r+1)(n/\lambda_0)^r)=\mathcal{O}\left({dn^{r/2}}\right)$ since $\lambda_0\ge \sqrt{n}$. Furthermore, the approximation error of the output of this exhaustive search algorithm is 
$$\max_{j\in[d-r+1]}\left|G(\bm{v}^*_j;\phi_j,\bm{\theta})-\max_{\bm{u}\in\mathbb{B}^{d}}G(\bm{u};\phi_j,\bm{\theta})\right|\le \frac{\lambda_0}{n}\sum_{i=1}^n\ell'(-y_i f(x_i;\bm{\theta}))\le \lambda_0.$$
Combining this inequality with Theorem~\ref{thm::algo}, we have the following Corollary.
\begin{corollary}\label{cor::1}
Let $m\ge (n+1)(d-r+1)$ and $\lambda_0=\sqrt{n}\ln n$. Assume that the data samples in the dataset are independently drawn from an underlying distribution $\mathbb{P}_{X\times Y}$. Under Assumption~\ref{assumption::loss-1} and \ref{assumption::dataset-1}, with probability at least $1-\delta$, the output $\bm{\theta}^*$ of  Algorithm~\ref{algo::main} satisfies
\begin{align*}
\mathbb{P}(Yf(X;\bm{\theta}^*)<0)=\mathcal{O}\left(\frac{E}{\gamma n}+\frac{\ln n}{\gamma \sqrt{n}}+\sqrt{\frac{\log(1/\delta)}{n}}\right).
\end{align*}
Furthermore, Algorithm~\ref{algo::main} has a computation complexity of $\mathcal{O}\left(dKm^2+dKn^{r/2}+nK^5/\lambda_0\right)$ where $K=\max\{\lceil L_n(\bm{\theta}(0);\bm{\lambda}(0)), 2n\rceil\}$.
\end{corollary}
\textbf{Remark:} The proof is provided in Appendix~\ref{appendix::proof::corollary-1}. We note here that for 
CNN with a given template size $r$, the complexity of Algorithm~\ref{algo::main} is of polynomial order of dimension $d$, samples size $n$ and the network size $m$. 
For instance, when the template size is $r = 3$
and $ m = (n+1) (d - 2 )$, 
the complexity is $ \mathcal{O}( K n^2 d^3 
 + d K n^{3/2} + \sqrt{n} K^5/\ln n ) .$
 For FNN where $r = d$,
 this complexity is not polynomial,
and we will discuss how to reduce the complexity
for FNN next. 

\subsection{Discussion: Improved Complexity for FNN}
From Corollary~\ref{cor::1}, we can see that when $r=d$ (i.e., the neural network is a fully connected neural network), the algorithm has a computation complexity of $\mathcal{O}(n^{d/2})$. This is due to the fact that solving the following optimization problem requires a computational complexity of $\mathcal{O}(n^{d/2})$, 
\begin{align}\label{eq::opt}
\max_{\bm{u}\in\mathbb{B}^d}\left|\sum_{i=1}^n \ell'(-y_if(x_i;\bm{\theta}))y_i (\bm{u}^\top x_i)_+\right|\end{align}
where the vector $\bm{\theta}$ is given. 
In this subsection, we will show that when the underlying data distribution has some additional properties, the above optimization problem can be solved by an algorithm of much smaller complexity with high probability.

\begin{assumption}\label{assump::telgarsky}
Let $\gamma \ge 8\e>0$. Assume that there exists a function $\bar{\bm{v}}(\omega):\mathbb{R}^d\rightarrow \mathbb{R}^d$ with $\|\bar{\bm{v}}(\omega)\|_2\le 1$  such that 
$$\mathbb{P}\left(Y\int\bar{\bm{v}}^\top({\omega})X\mathbb{I}\{X^\top \omega\ge 0\}\mu(d\omega)\ge \gamma\right)\ge 1-\varepsilon,$$
where $\mu$ is the uniform measure defined on the sphere $\mathbb{S}^{d-1}$.
\end{assumption}

\begin{algorithm}[t]\label{algo::opt}
\SetAlgoLined
{
\tcp{Fast Solver of Problem~\eqref{eq::opt}}
\tcp{At the initialization of Algorithm~\ref{algo::main}}
  Choose $\bm{w}_1,...,\bm{w}_{2M}$ from the sphere $\mathbb{S}^{d-1}$ uniformly at random and choose $r>0$\;
 \tcp{At each time of running Algorithm~\ref{algo::perturb}, perturb the inactive neuron in the following direction.}
For each $j\in[2M]$, calculate $$\bm{v}_j=\frac{\sum_{i=1}^n y_i\ell'(-y_i f(x_i;\bm{\theta})) x_i \mathbb{I}\{\bm{\omega}_j^\top x_i\ge 0\}}{\left\|\sum_{i=1}^n y_i\ell'(-y_i f(x_i;\bm{\theta})) x_i \mathbb{I}\{\bm{\omega}_j^\top x_i\ge 0\}\right\|_2}.$$
 }
 Then $\bm{v}^*=\bm{\omega}_{j^*} + ra_j \bm{v}_{j^*}^*$ where 
 $$j^*=\arg\max_j \left|\sum_{i=1}^n y_i\ell'(-y_i f(x_i;\bm{\theta}))\left[(\bm{\omega}_j+ra_j \bm{v}^*_j)^\top x_i\right]_+\right|$$
 \caption{Finding The Optimal Descent Direction $\bm{v}^*$}
\end{algorithm}
Now we will show that under Assumption~\ref{assump::telgarsky}, the optimization problem defined by \eqref{eq::opt} can be solved by  Algorithm~\ref{algo::opt}. The performance of the above algorithm is guaranteed by the following lemma. 
\begin{lemma}\label{lemma::algo-performance}
Let $\e_0=\max\{\e, n^{-1/3}\}$ and  $\bm{\omega}_1,...,\bm{\omega}_{2M}$ be i.i.d. random vectors uniformly distributed on the sphere $\mathbb{S}^{d-1}$. Under Assumption~\ref{assump::telgarsky}, if $$r<\frac{\gamma-4\e_0}{16d},\quad n\ge \frac{\ln (6/\delta)}{2\e_0^2}\quad\text{and}\quad M\ge \max\left\{\frac{\ln(4n/\delta)}{\gamma^2},\frac{4\ln(6n/\delta)}{r^2(\gamma-4\e_0)^2}\right\},$$ then 
\begin{align*}
\mathbb{P}\left( \forall \bm{\beta}\in[0,1]^n:\max_{j\in[2M]}\left|\sum_{i=1}^n y_i\beta_i\left[(\bm{\omega}_j+ra_j \bm{v}_j)^\top x_i\right]_+\right|\ge \frac{r(\gamma-4\e_0)}{8}\max_{\bm{u}\in\mathbb{B}^d}\left|\sum_{i=1}^ny_i\beta_i(\bm{u}^\top x_i)\right|\right)\ge 1-\delta
\end{align*}
where 
$$\bm{v}_j=\frac{\sum_{i=1}^n y_i\beta_i x_i \mathbb{I}\{\bm{\omega}_j^\top x_i\ge 0\}}{\left\|\sum_{i=1}^n y_i\beta_i x_i \mathbb{I}\{\bm{\omega}_j^\top x_i\ge 0\}\right\|_2}.$$
\end{lemma}
\textbf{Remark: }Lemma~\ref{lemma::algo-performance} ensures that if the radius $r$ is sufficiently small, the sample size $n$ is sufficiently large and the number of random vector $M$
is sufficiently large, then with high probability, the $v^*$ found by Algorithm~\ref{algo::opt} is an approximate solution of the problem~\eqref{eq::opt}.
\begin{corollary}\label{cor::2}
Let $m\ge n+1$ and $\lambda_0=\sqrt{n}\ln n$. Assume that the data samples in the dataset are independently drawn from an underlying distribution $\mathbb{P}_{X\times Y}$. Under Assumption~\ref{assumption::loss-1} and \ref{assump::telgarsky}, with probability at least $1-\delta$, the output $\bm{\theta}^*$ of  Algorithm~\ref{algo::main} satisfies
\begin{align*}
\mathbb{P}(Yf(X;\bm{\theta}^*)<0)=\mathcal{O}\left(\frac{\e}{\gamma }+\frac{\ln n}{r\gamma^2 \sqrt{n}}+\sqrt{\frac{\log(1/\delta)}{n}}\right).
\end{align*}
Furthermore, Algorithm~\ref{algo::main} has a computation complexity of $\mathcal{O}(dKm^2+ MK + nK^5/\lambda_0).$ where $K=\max\{\lceil L_n(\bm{\theta}(0);\bm{\lambda}(0)), 2n\rceil\}$.
\end{corollary}

\begin{proof}
The upper bound on $\mathbb{P}(Yf(X;\bm{\theta}^*)<0)$ follows directly from Theorem~\ref{thm::other-gen}, with $C=\frac{40\lambda_0}{r\gamma}$, $\e = \frac{\lambda_0}{16K\sqrt{n}}$ and $\|\bm{a}^*\|_2\le 2\sqrt{\frac{K}{\lambda_0}}$, where the upper bound on $\e$ and $\|\bm{a}^*\|_2$ follows from Claim~\ref{claim::algo::2} in Appendix~\ref{appendix::proof::thm::other-gen}. Furthermore, following the same analysis on the computational complexity in the proof of Corollary~\ref{cor::1}, we know that the complexity of choosing the regularization coefficient is $\mathcal{O}(dm^2)$ and the complexity of running gradient descent is $\mathcal{O}(nK^4/\lambda_0)$. Furthermore, it is easy to see that running Algorithm~\ref{algo::opt} requires time complexity of $\mathcal{O}(M)$. Since there are $K$ iterations in total, then the time complexity of the whole algorithm is $\mathcal{O}(dKm^2+ MK + nK^5/\lambda_0).$
\end{proof}

\section{Conclusions}

We consider a single hidden layer overparameterized neural network trained for binary classification. Our main results characterize local minima in certain directions which have good generalization performance and present a specialized version of gradient descent to find such points. Unlike much of the prior work on this topic, we are able to handle datasets which are not separable through the use of explicit $\ell_2$ regularization. Interestingly, the amount of overparameterization we need is just one more than the number of points in the dataset. While we do not consider the robustness of the output of the neural network to small changes in the input, it is interesting that the number of neurons we require for the landscape properties to hold in fully connected neural networks is just one more than the number conjectured to be required in \cite{bubeck2020law} for such type of robustness.


\bibliographystyle{plain}
\bibliography{refs}

\begin{appendix}

\section{Proof of Theorem~\ref{thm::other-memo}}\label{appendix::proof::thm::other-memo}

\begin{proof}
If $\bm{\theta}^*$ satisfies the condition presented in Theorem~\ref{thm::other-memo}, then we have for any $k=1,...,d-r+1$
$$\max_{\bm{u}\in\mathbb{B}^d}\left|\sum_{i=1}^n\ell'(-y_i f(x_i;\bm{\theta}))y_i(\bm{u}^\top (x_i\odot \phi_k))_+\right|\le C$$
holds for some positive number $C>0$. From Assumption~\ref{assumption::dataset-1}, there exists an integer $m\ge 1$ and a single-layered ReLU network $h\in\mathcal{H}$ of size $m,$ such that $\sum_{i=1}^n\mathbb{I}\{y_ih(x_i)\ge \gamma\}\ge n-E$. 
Denote a set $ \Omega 
 = \{ i \mid y_ih(x_i)\ge \gamma \} $, 
 then $|\Omega| \geq n - E. $
Then
\begin{align*}
&\left|\sum_{i=1}^{n}\ell'(-y_{i}f(x_{i};\bm{\theta}^{*}))y_{i}h(x_i)\right|\ge \sum_{i=1}^{n}\ell'(-y_{i}f(x_{i};\bm{\theta}^{*}))y_{i}h(x_i)\\
&= \sum_{i=1}^{n}\ell'(-y_{i}f(x_{i};\bm{\theta}^{*}))y_{i}h(x_i)\mathbb{I}\{y_i h(x_i)\ge \gamma\}+\sum_{i=1}^{n}\ell'(-y_{i}f(x_{i};\bm{\theta}^{*}))y_{i}h(x_i)\mathbb{I}\{y_i h(x_i)< \gamma\}\\
& \ge 
\gamma\sum_{i=1}^{n}\ell'(-y_{i}f(x_{i};\bm{\theta}^{*}))\mathbb{I}\{y_i h(x_i)\ge \gamma\}
+
\sum_{i=1}^{n} \ell'(-y_{i}f(x_{i};\bm{\theta}^{*}))
y_{i}h(x_i)\mathbb{I}\{y_i h(x_i)< \gamma\}
\\
&  =
  \gamma
  [ \sum_{i=1 }^n \ell'(-y_{i}f(x_{i};\bm{\theta}^{*})) - 
  \sum_{i \notin \Omega } \ell'(-y_{i}f(x_{i};\bm{\theta}^{*}))   ]
  - 
  \sum_{i \notin \Omega } \ell'(-y_{i}f(x_{i};\bm{\theta}^{*})) y_i h(x_i)
\\
&  =
  \gamma
\sum_{i=1 }^n \ell'(-y_{i}f(x_{i};\bm{\theta}^{*})) - 
  \sum_{i \notin \Omega } \ell'(-y_{i}f(x_{i};\bm{\theta}^{*}))  (\gamma 
  +   y_i h(x_i) )
\\
& \ge \gamma\sum_{i=1}^{n}\ell'(-y_{i}f(x_{i};\bm{\theta}^{*}))
-
E(1+\gamma) {\color{blue} or: E(\gamma + \gamma )}
\ge \gamma\sum_{i=1}^{n}
\ell'(-y_{i}f(x_{i};\bm{\theta}^{*})) 
 -2 E
\end{align*}
In the second last inequality,
we used the fact that $\ell'(t) \leq 1
\; \forall \; t$;
in the final inequality, we used
the assumption $\gamma \leq 1. $
Since $h\in\mathcal{H}$, then we can write $h$ as $h(x)=\sum_{j=1}^m c_j(\bm{v}_j^\top (x\odot \phi_j))_+$ for $\sum_{j=1}^m|c_j|=1$ and $\|\bm{v}_j\|_2=1$, $j=1,...,m$. Then we have 
\begin{align*}
&\left|\sum_{i=1}^{n}\ell'(-y_{i}f(x_{i};\bm{\theta}^{*}))y_{i}h(x_i)\right|=\left|\sum_{i=1}^{n}\ell'(-y_{i}f(x_{i};\bm{\theta}^{*}))y_{i}\left(\sum_{j=1}^m c_j(\bm{v}_j^\top (x_i\odot \phi_j))_+\right)\right|\\
&=\left|\sum_{j=1}^m c_j\left(\sum_{i=1}^n\ell'(-y_{i}f(x_{i};\bm{\theta}^{*}))y_{i}(\bm{v}_j^\top (x_i\odot \phi_j)_+\right)\right|\le \sum_{j=1}^m |c_j|\cdot\left|\sum_{i=1}^n\ell'(-y_{i}f(x_{i};\bm{\theta}^{*}))y_{i}(\bm{v}_j^\top (x_i\odot \phi_j)_+\right|\\
&\le C\sum_{j=1}^m |c_j|=C
\end{align*}

Combining the above two inequalities, we should have 
\begin{equation}\label{eq::other-memo::1}\sum_{i=1}^{n}\ell'(-y_{i}f(x_{i};\bm{\theta}^{*}))\le (C+2E)/\gamma.\end{equation}
Furthermore, by the convexity of the loss $\ell$, we should have 
$\ell'(z)\ge \ell'(0)\mathbb{I}\{z\ge 0\}$, which further indicates 
$$R_n(f;\bm{\theta}^*)=\frac{1}{n}\sum_{i=1}^n\mathbb{I}\{y_i\neq \sgn(f(x_i;\bm{\theta}^*))\}\le \frac{1}{\ell'(0)n}\sum_{i=1}^{n}\ell'(-y_{i}f(x_{i};\bm{\theta}^{*}))\le \frac{C+2E}{ \ell'(0)\gamma n}.$$
\end{proof}

\section{Proof of Theorem~\ref{thm::other-gen}}\label{appendix::proof::thm::other-gen}

\begin{proof}
\textbf{(1)} We first bound $\sum_{j=1}^m(\|\bm{w}^*_j\|_2^2+{a_j^*}^2)$. 
Based on the proof of Theorem~\ref{thm::other-memo}, from equality~\eqref{eq::other-memo::1}, we have 
$$\sum_{i=1}^{n}\ell'(-y_{i}f(x_{i};\bm{\theta}^{*}))\le (C+2E)/\gamma.$$ 
Now we bound the parameter norm $\|\bm{\theta}^*\|_2$. By the first condition, we have 
$$\|\nabla_{\bm{a}}L_n(\bm{\theta}^*)\|_{2}\le \varepsilon.$$
Further, since 
\begin{align*}
\e\|\bm{a}^*\|_2\ge|\langle\bm{a}^*, \nabla_{\bm{a}}L_n(\bm{\theta}^*)\rangle|&=\left|-\sum_{i=1}^n \ell'(-y_i f(x_i;\bm{\theta}^*))y_i f(x_i;\bm{\theta}^*)+\sum_{j=1}^m\lambda_j (a^*_j)^2\right|\\
&\ge \sum_{j=1}^m\lambda_j (a^*_j)^2 -\sum_{i=1}^n \ell'(-y_i f(x_i;\bm{\theta}^*))y_i f(x_i;\bm{\theta}^*)
\end{align*}
then 
$$\sum_{j=1}^m\lambda_j (a^*_j)^2\le \sum_{i=1}^n \ell'(-y_i f(x_i;\bm{\theta}^*))y_i f(x_i;\bm{\theta}^*)+\e\|\bm{a}^*\|_2.$$
Further, by assumption~\ref{assumption::loss-1}, we assume that there exists a positive number $a>0$ such that $\ell'(z)\le e^{az}$ holds for any $z\in\mathbb{R}$. Then for each $i=1,...,n$,
$$0\le \ell'(-y_i f(x_i;\bm{\theta}^*))\le e^{-ay_if(x_i;\bm{\theta}^*)}$$
which further indicates 
$$y_if(x_i;\bm{\theta}^*)\le \frac{1}{a}\ln\left(\frac{1}{\ell'(-y_i f(x_i;\bm{\theta}^*))}\right),$$
and thus 
$$\sum_{j=1}^m\lambda_j (a^*_j)^2\le \frac{1}{a}\sum_{i=1}^n \ell'(-y_i f(x_i;\bm{\theta}^*))\ln\left(\frac{1}{\ell'(-y_i f(x_i;\bm{\theta}^*))}\right)+\e\|\bm{a}^*\|_2.$$
By the concavity of function $-x\ln x$ and Jensen's inequality, we have 
$$\frac{1}{n}\sum_{i=1}^n \ell'(-y_i f(x_i;\bm{\theta}^*))\ln\left(\frac{1}{\ell'(-y_i f(x_i;\bm{\theta}^*))}\right)\le \frac{\sum_{i=1}^n \ell'(-y_i f(x_i;\bm{\theta}^*))}{n}\ln\left(\frac{n}{\sum_{i=1}^n \ell'(-y_i f(x_i;\bm{\theta}^*))}\right).$$
Furthermore, the function $-x\ln x$ is increasing on the region $[0,e^{-1}]$, and 
$$\frac{1}{n}\sum_{i=1}^{n}\ell'(-y_{i}f(x_{i};\bm{\theta}^{*}))\le \frac{C+2E}{\gamma n}\le e^{-1}$$
then 
$$\frac{\sum_{i=1}^n \ell'(-y_i f(x_i;\bm{\theta}^*))}{n}\ln\left(\frac{n}{\sum_{i=1}^n \ell'(-y_i f(x_i;\bm{\theta}^*))}\right)\le \frac{C+2e}{\gamma n}\ln(\frac{\gamma n}{C+2e})\le \frac{(C+2E)\ln n}{\gamma n}.$$
we have 
\begin{align*}
 \sum_{j=1}^m\lambda_j (a^*_j)^2&\le \frac{1}{a}\sum_{i=1}^n \ell'(-y_i f(x_i;\bm{\theta}^*))\ln\left(\frac{1}{\ell'(-y_i f(x_i;\bm{\theta}^*))}\right)+\e\|\bm{a}^*\|_2\\
 &\le\frac{\sum_{i=1}^n \ell'(-y_i f(x_i;\bm{\theta}^*))}{a}\ln\left(\frac{n}{\sum_{i=1}^n \ell'(-y_i f(x_i;\bm{\theta}^*))}\right)+\e\|\bm{a}^*\|_2\\
 &\le \frac{(C+2E)\ln n}{a\gamma } +\e\|\bm{a}^*\|_2
\end{align*}
Since $\min_j\lambda_j\ge \lambda_0/2$, then we have 
$$\sum_{j=1}^m {a_j^*}^2\le \frac{(2C+4E)\ln n}{a\gamma \lambda_0} +\frac{2\e\|\bm{a}^*\|_2}{\lambda_0}.$$
Since we assume that 
$$\sum_{j=1}^m \|\bm{w}^*_j\|_2^2= \sum_{j=1}^m {a_j^*}^2,$$
then we have 
$$\sum_{j=1}^m(\|\bm{w}^*_j\|_2^2+{a_j^*}^2)\le \frac{(6C+12E)\ln n}{a\gamma \lambda_0} +\frac{6\e\|\bm{a}^*\|_2}{\lambda_0}$$

\textbf{(2)} We define the following function class
$$\mathcal{F}_c=\bigcup_{m\ge 1}\left\{x\mapsto \sum_{j=1}^ma_j(\bm{w}_j^\top(x\odot \phi_j)_+\Bigg|\sum_{j=1}^m a_j^2\le c, \|\bm{w}_j\|_2\le |a_j|\text{ for }j=1,...,m\right\}.$$
It is straightforward to show that 
$$\mathcal{F}_c=\bigcup_{m\ge 1}\left\{x\mapsto \sum_{j=1}^ma_j|a_j|(\bm{w}_j^\top (x\odot \phi_j)_+\Bigg|\sum_{j=1}^m a_j^2\le c, \|\bm{w}_j\|_2\le 1\text{ for }j=1,...,m\right\}.$$
Now we bound the Rademacher complexity  $R_n(\mathcal{F}_c)$ under the dataset $\mathcal{D}$. By definition, the Rademacher complexity is 
$$R_{n}(\mathcal{F}_c)=\mathbb{E}_{\sigma}\left[\sup_{f\in\mathcal{F}}\frac{1}{n}\left|\sum_{i=1}^{n}\sigma_{i}f(x_{i})\right|\right],$$
where $\{\sigma_{i}\}$ are a sequence of Rademacher random variables.
Thus, we have 
\begin{align*}
R_{n}(\mathcal{F}_c)&=\frac{1}{n}\mathbb{E}_{\sigma}\left[\sup_{f\in\mathcal{F}_c}\left|\sum_{i=1}^{n}\sigma_{i}f(x_{i})\right|\right]\\
&=\frac{1}{n}\mathbb{E}_{\sigma}\left[\sup_{\|\bm{a}\|_2^2\le c,\bm{w}_{j}\in\mathbb{B}^{d}}\left|\sum_{i=1}^{n}\sigma_{i}\left(\sum_{j=1}^{m}a_{j}|a_j|(\bm{w}_{j}^{\top}(x_i\odot \phi_j))_{+}^{}\right)\right|\right]\\
&=\frac{1}{n}\mathbb{E}_{\sigma}\left[\sup_{\|\bm{a}\|_2^2\le c,\bm{w}_{j}\in\mathbb{B}^{d}}\left|\sum_{j=1}^{m}a_{j}|a_j|\sum_{i=1}^{n}\sigma_{i}(\bm{w}_{j}^{\top}(x_i\odot \phi_j))_{+}\right|\right]\\
\end{align*}
Since 
\begin{align*}
\sup_{\|\bm{a}\|_2^2\le c,\bm{w}_{j}\in\mathbb{B}^{d}}\left|\sum_{j=1}^{m}a_{j}|a_j|\sum_{i=1}^{n}\sigma_{i}(\bm{w}_{j}^{\top}(x_i\odot \phi_j))_{+}\right|&\le \sup_{\|\bm{a}\|_2^2\le c}\sum_{j=1}^{m}\left[|a_j|^2\sup_{\bm{w}_{j}\in\mathbb{B}^{d}}\left|\sum_{i=1}^{n}\sigma_{i}(\bm{w}_{j}^{\top}(x_i\odot \phi_j))_{+}\right|\right]\\
&=c\sup_{\bm{w}\in\mathbb{B}^{d}}\left|\sum_{i=1}^{n}\sigma_{i}(\bm{w}^{\top}(x_i\odot \phi_j))_{+}\right|
\end{align*}
Thus, 
$$R_{n}(\mathcal{F}_c)\le c\mathbb{E}_\sigma\left[\sup_{\bm{w}\in\mathbb{B}^{d}}\frac{1}{n}\left|\sum_{i=1}^{n}\sigma_{i}(\bm{w}^{\top}(x_i\odot \phi_j))_{+}\right|\right]\le c\mathbb{E}_\sigma\left[\sup_{\bm{w}\in\mathbb{B}^{d}}\frac{1}{n}\left|\sum_{i=1}^{n}\sigma_{i}(\bm{w}^{\top}(x_i\odot \phi_j))\right|\right]\le c\sqrt{2/n},$$
which further indicates that 
$$\mathbb{E}R_{n}(\mathcal{F}_c)\le c\sqrt{2/n}.$$
Setting $c= \frac{(6C+12E)\ln n}{a\gamma \lambda_0} +\frac{6\e\|\bm{a}^*\|_2}{\lambda_0}$ and using the results in part (1), we know that for any points $\bm{\theta}^*$ satisfying all conditions in Theorem~\ref{thm::other-gen}, the Rademacher complexity is upper bounded by $$\mathbb{E}R_{n}(\mathcal{F}_c)\le \left(\frac{(6C+12E)\ln n}{a\gamma \lambda_0} +\frac{6\e\|\bm{a}^*\|_2}{\lambda_0}\right)\sqrt{2/n}.$$

\textbf{(3)} Finally, we prove the generalization bound. Since inequalities $\ell'(z)/\ell'(0)\le 1/\ell'(0)$ and $\ell'(z)/\ell'(0)\ge \mathbb{I}(z\ge 0)$ hold for all $z\in\mathbb{R}$ and the function $\ell'(z)/\ell'(0)$ is $1/\ell'(0)$-Lipschitz, then for any $\delta\in(0,1)$, the following bound holds with probability at least $1-\delta$:
$$\mathbb{P}(Yf(X;\bm{\theta}^*)<0)\le \frac{1}{n\ell'(0)}\sum_{i=1}^n\ell'(-y_i f(x_i;\bm{\theta}^*))+\frac{4\mathbb{E}R_n(\mathcal{F}_c)}{\ell'(0)}+\frac{1}{\ell'(0)}\sqrt{\frac{\log(1/\delta)}{2n}}.$$
Using the results of Theorem~\ref{thm::other-memo} and results of part (2), we have 
\begin{align*}
   \mathbb{P}(Yf(X;\bm{\theta}^*)<0)&\le \frac{C+2E}{\gamma n\ell'(0)}+\left(\frac{(6C+12E)\ln n}{a\gamma \lambda_0} +\frac{6\e\|\bm{a}^*\|_2}{\lambda_0}\right)\frac{4\sqrt{2}}{\ell'(0)\sqrt{n}}+\frac{1}{\ell'(0)}\sqrt{\frac{\log(1/\delta)}{2n}}\\ &=\mathcal{O}\left(\frac{\lambda_0+C+E}{\gamma {n}}+\frac{(C+E)\ln n}{\gamma \lambda_0\sqrt{n}}+\frac{\e \|\bm{a}^*\|_2}{\lambda_0\sqrt{n}}+\sqrt{\frac{\log(1/\delta)}{n}}\right)
\end{align*}
\end{proof}

\section{Proof of Theorem ~\ref{thm::single}}\label{appendix::proof::memo}
The proof of Theorem~\ref{thm::single} is based on the proof of Theorem~\ref{thm::other-memo}. 

\subsection{Important Lemma}
In this subsection, we present an important lemma. The following lemma is consisted of four parts. In the first part, we show that at every local minimum, the parameters of each neuron is balanced. In the second part, we present the first order condition for $\bm{w}_j$'s since the empirical loss is not directly differentiable with respect to $\bm{w}_j$'s. In the third part, we show that 
there is a zero measure set of 
the regularizer coefficient vector $\bm{\lambda}$ 
such that there exists
a local minimum where all neurons are active.
Finally, we show that for almost all regularizer coefficient vectors, at every local minimum of the empirical loss, the neural network always has an inactive neuron.

\begin{lemma}\label{lemma::single-1}
Assume $m\ge n+1$. Let $\bm{\theta}^*$ be a local minimum of the empirical loss $L_{n}(\bm{\theta};\bm{\lambda})$, then the following statements are true:
\begin{itemize}[leftmargin=*]
\item[(1)] the equation ${a_j^*}^2 = \|\bm{w}_j^*\|^2_2 $ holds for any $j\in[m]$.
\item[(2)] let $\ell'_i\triangleq\ell'(-y_if(x_i;\bm{\theta}^*))$ and $I_j \triangleq\{i\in[n]:{\bm{w}^*_j}^\top (x_i\odot \phi_j) =0\}$, then
$$\lambda_j \bm{w}_j^*-\sum_{i=1}^{n}\ell'_i y_i a_j^* \mathbb{I}\{{\bm{w}^*}^\top (x_i\odot \phi_j)>0\}x_i\in \left\{\sum_{i\in I_j}\mu_{ji}\ell'_i y_i a_j^* (x_i\odot \phi_j) \Bigg| \mu_{ji}\in[0,1]\right\}$$
holds for any $j\in[m]$.
\item[(3)] if the local minimum is consisted of all active neurons, i.e., $a^*_j>0$ for all $j\in[m]$, then there exists a finite function class $\mathcal{G}$ consisted of all globally Lipschitz functions such that 
$$\bm{\lambda}=(\lambda_1,...,\lambda_m)\in \bigcup_{\bm{\Lambda}\in\mathcal{G}}\bm{\Lambda}(\mathbb{R}^n).$$
\item[(4)] there exists a zero measure set $\mathcal{C}(\mathcal{D})\subset\mathbb{R}^{m}$ depending on the dataset $\mathcal{D}$ such that for any $\bm{\lambda}\notin \mathcal{C}$, at every local minimum  $\bm{\theta}^{*}$ of the empirical loss $L_{n}(\bm{\theta};\bm{\lambda})$, the neural network $f(x;\bm{\theta}^{*})$ always has an inactive neuron, i.e., $\exists j\in[m]$ s.t. $(a^{*}_{j}, \|\bm{w}_{j}^{*}\|_{2})=(0,0)$.
\end{itemize}
\end{lemma}

\subsection{Proof  of Lemma~\ref{lemma::single-1}}
\begin{proof}
\textbf{(1)} Given a set of parameters $\bm{\theta}^*$, we define the following function for each neuron $j\in[m]$, 
$$H_j(v_1, v_2;\bm{\theta}^*)={L}_n(\bm{\theta}_{-j}^*, v_1a_j^*, v_2\bm{w}_j^*),$$
where $\bm{\theta}_{-j}^*$ denote the vector consisted of all parameters in the neural network except the parameters of the $j$-th neuron, i.e., $$\bm{\theta}_{-j}^*=(a^*_1,...,a^*_{j-1},a^*_{j+1},...,a^*_m,\bm{w}^*_1,...,\bm{w}^*_{j-1},\bm{w}^*_{j+1},...,\bm{w}^*_m).$$ We note here that for simiplicity of notation, we define 
$$f_{-j}(x;\bm{\theta})=\sum_{k\neq j}a_k(\bm{w}_k^\top (x\odot \phi_k))_+$$
and 
$$V_{-j}(\bm{\theta})=\frac{1}{2}\sum_{k\neq j}(a_k^2+\|\bm{w}_k\|_2^2),$$
to be the output of the neural network and the regularizer without considering the $j$-th neuron. 
Since the ReLU is positive homogeneous, i.e., $(v_2z)_+\equiv v_2(z)_+$  for any $v_2>0$, then 
\begin{align*}
H_j(v_1, v_2;\bm{\theta}^*)&=\sum_{i\neq j}\ell(-y_i f_{-j}(x_i;\bm{\theta}^*)-y_iv_1a_j^*(v_2{\bm{w}_j^*}^\top (x_i\odot \phi_j))_+)\\
&\quad+V_{-j}(\bm{\theta}^*)+\frac{\lambda_j}{2}(v_1^2 {a^*_j}^2+v_2^2\|\bm{w}_j^*\|_2^2)\\
&=\sum_{i\neq j}\ell(-y_i f_{-j}(x_i;\bm{\theta}^*)-y_iv_1v_2a_j^*({\bm{w}_j^*}^\top (x_i\odot \phi_j))_+)\\
&\quad+V_{-j}(\bm{\theta}^*)+\frac{\lambda_j}{2}(v_1^2 {a^*_j}^2+v_2^2\|\bm{w}_j^*\|_2^2)
\end{align*}
Now, it is easy to see that $H_j(v_1, v_2;\bm{\theta}^*)$ is differentiable  on $\mathbb{R}\times \mathbb{R}_+$. Furthermore, since $\bm{\theta}^*$ is a local minimum of the empirical loss function $L_n(\bm{\theta};\bm{\lambda})$, then it is easy to see that $(v_1, v_2)=(1, 1)$ should also be a local minimum of the function $H_j$. Therefore, we should have 
$$\frac{\partial H_j(1, 1)}{\partial v_1}=0 \quad\text{and} \quad\frac{\partial H_j(1, 1)}{\partial v_2}=0.$$ Therefore, we have 
\begin{align*}
0&=\frac{\partial H_j(1, 1)}{\partial v_1}=\sum_{i=1}^n\ell'(-y_i f(x_i;\bm{\theta}^*))(-y_i)a_j^*({\bm{w}_j^*}^\top (x_i\odot \phi_j) )_+ +\lambda_j {a_j^*}^2,\\
0&=\frac{\partial H_j(1, 1)}{\partial v_2}=\sum_{i=1}^n\ell'(-y_i f(x_i;\bm{\theta}^*))(-y_i)a_j^*({\bm{w}_j^*}^\top (x_i\odot \phi_j) )_+ +\lambda_j \|\bm{w}_j^*\|^2_2 .
\end{align*}
By comparing the above results, we conclude that 
$${a_j^*}^2 = \|\bm{w}_j^*\|^2_2.$$

\textbf{(2)} If $\bm{\theta}^*$ is a local minimum, then by definition, there exists $\delta>0$ such that for any perturbed parameter $\tilde{\bm{\theta}}:\|\tilde{\bm{\theta}}-\bm{\theta}^*\|_2\le \delta $, we should have 
$$L(\tilde{\bm{\theta}})\ge L(\bm{\theta}^*).$$ However, $L_n$ is not always differentiable with respect to $w_j$'s, so we cannot  use the first order condition directly. Now we perturb the $j$-th neuron  from $(a_j^*, \bm{w}_j^*)$ to $(a_j^*, \bm{w}_j^*+\delta \bm{u})$  with $\|\bm{u}\|_2^2=1$. Therefore, we have 
\begin{align*}
L_n(\tilde{\bm{\theta}})&=\sum_{i=1}^n\ell\left(-y_i f_{-j}(x_i;\bm{\theta}^*)-y_ia_j^*({\bm{w}_j^*}^\top (x_i\odot \phi_j) +\delta \bm{u}^\top (x_i\odot \phi_j))_+\right) +V_{-j}(\bm{\theta}^*)+\frac{\lambda_j}{2}\left({a_j^*}^2+\|\bm{w}_j^*+\delta\bm{u}\|_2^2\right)\\
&\ge L_n(\bm{\theta}^*)
\end{align*}

Therefore, we have for any $(\bm{u}, v):\|\bm{u}\|_2^2=1$, 
\begin{align*}
\sum_{i=1}^n\ell'(-y_i f(x_i;\bm{\theta}^*))(-y_i)a_j^*&\left[({\bm{w}_j^*}^\top (x_i\odot \phi_j)+\delta \bm{u}^\top (x_i\odot \phi_j) )_+-({\bm{w}^*}^\top (x_i\odot \phi_j))_+\right]+\delta \lambda_j \left({\bm{w}_j^*}^\top \bm{u} \right)\ge 0
\end{align*}

Now we choose $\delta$ to be so sufficiently small such that $$0<\delta < \min_{i, j:|{\bm{w}_j^*}^\top (x_i\odot \phi_j)|\neq 0}|{\bm{w}_j^*}^\top (x_i\odot \phi_j)|.$$ Therefore, when $\delta$ is sufficiently small, we have 
\begin{align*}
&\left({\bm{w}^*_j}^\top (x_i\odot \phi_j) + +\delta \bm{u}^\top (x_i\odot \phi_j)+\delta v\right)_+ -\left({\bm{w}^*_j}^\top (x_i\odot \phi_j) \right)_+ \\
&= \delta\mathbb{I}\left\{{\bm{w}^*_j}^\top (x_i\odot \phi_j)> 0\right\}(\bm{u}^\top (x_i\odot \phi_j)+v)+\delta \mathbb{I}\left\{{\bm{w}^*_j}^\top (x_i\odot \phi_j)= 0\right\}(\bm{u}^\top (x_i\odot \phi_j)+v)_+.
\end{align*}
Thus, when $\delta$ is sufficiently small, we have for any $(\bm{u}, v):\|\bm{u}\|_2^2+v^2=1$,
\begin{align*}
&\delta\sum_{i=1}^n\ell'(-y_i f(x_i;\bm{\theta}^*))(-y_i)a_j^*\mathbb{I}\left\{{\bm{w}^*_j}^\top (x_i\odot \phi_j)> 0\right\}(\bm{u}^\top (x_i\odot \phi_j))+\delta \lambda_j \left({\bm{w}_j^*}^\top \bm{u} \right)\\
&\quad+\delta\sum_{i=1}^n\ell'(-y_i f(x_i;\bm{\theta}^*))(-y_i)a_j^*\mathbb{I}\left\{{\bm{w}^*_j}^\top (x_i\odot \phi_j)= 0\right\}(\bm{u}^\top (x_i\odot \phi_j))_+\ge 0,
\end{align*}
or 
\begin{align}
&\sum_{i=1}^n\ell'(-y_i f(x_i;\bm{\theta}^*))(-y_i)a_j^*\mathbb{I}\left\{{\bm{w}^*_j}^\top (x_i\odot \phi_j)> 0\right\}(\bm{u}^\top (x_i\odot \phi_j))+ \lambda_j \left({\bm{w}_j^*}^\top \bm{u} \right)\notag\\
&\quad+\sum_{i=1}^n\ell'(-y_i f(x_i;\bm{\theta}^*))(-y_i)a_j^*\mathbb{I}\left\{{\bm{w}^*_j}^\top (x_i\odot \phi_j)= 0\right\}(\bm{u}^\top (x_i\odot \phi_j))_+\ge 0.\label{eq::lemma2-1}
\end{align}
For simplicity of notation, for a given index $j$, we define the vector $\bm{p}_j\in\mathbb{R}^d$ as 
$$\bm{p}_j\triangleq\sum_{i=1}^n\ell'(-y_i f(x_i;\bm{\theta}^*))(-y_i)a_j^*\mathbb{I}\left\{{\bm{w}^*_j}^\top (x_i\odot \phi_j)> 0\right\}x_i+ \lambda_j \bm{w}_j^*$$
and a series of scalars $z_{ji}$'s as
$$z_{ji} = \ell'(-y_i f(x_i;\bm{\theta}^*))y_i a_j^*.$$
Now, the inequality \eqref{eq::lemma2-1} becomes, for any $(\bm{u}, v):\|\bm{u}\|_2^2=1$,
\begin{equation}\label{eq::lemma2-3}
    \bm{p}_j^\top \bm{u}  -\sum_{i=1}^{n}z_{ji} \mathbb{I}\left\{{\bm{w}^*_j}^\top (x_i\odot \phi_j)= 0\right\}(\bm{u}^\top (x_i\odot \phi_j))_+\ge 0.
\end{equation}
We finish the proof by contradiction.
We assume that 
$$\bm{p}_j\notin \left\{\sum_{i=1}^{n}\mu_{ji}z_{ji} \mathbb{I}\left\{{\bm{w}^*_j}^\top (x_i\odot \phi_j)= 0\right\}x_i\Bigg| \mu_{ji}\in [0,1] \right\}\triangleq \mathcal{D}_j.$$
Since the sets $\left\{\bm{p}_j\right\}$ and  $\mathcal{D}_j$ are non-empty disjoint convex sets and the set  $\left\{\bm{p}_j\right\}$ is compact, then by the hyperplane separation theorem, there exists a unit vector $\bm{u}_0$ satisfying 
\begin{equation}\label{eq::lemma2-2}
\bm{p}_j^\top \bm{u}_0 - \bm{r}^\top \bm{u}_0 <0,\quad \text{for all }\bm{r}\in \mathcal{D}_j.
\end{equation}
Now  setting 
$$\bm{r}=\sum_{i=1}^{n}\mathbb{I}\{\bm{u}_0^\top (x_i\odot \phi_j) \ge 0\}z_{ji} \mathbb{I}\left\{{\bm{w}^*_j}^\top (x_i\odot \phi_j)= 0\right\}x_i\in \mathcal{D}_j,$$
and substituting it into the inequality \eqref{eq::lemma2-2}, we have 
\begin{equation}
\bm{p}_j^\top \bm{u}_0  - \sum_{i=1}^{n}z_{ji} \mathbb{I}\left\{{\bm{w}^*_j}^\top (x_i\odot \phi_j)= 0\right\}(\bm{u}_0^\top (x_i\odot \phi_j) )_+<0
\end{equation}
which contradicts with the inequality \eqref{eq::lemma2-3}. Therefore, 
$$\bm{p}_j\in \left\{\sum_{i=1}^{n}\mu_{ji}z_{ji} \mathbb{I}\left\{{\bm{w}^*_j}^\top (x_i\odot \phi_j)= 0\right\}x_i\Bigg| \mu_{ji}\in [0,1] \right\}.$$
Recalling the definition of $\bm{p}_j$ and $z_{ji}$ and defining  $\ell'_i\triangleq\ell'(-y_if(x_i;\bm{\theta}^*))$ and $I_j \triangleq\{i\in[n]:{\bm{w}^*_j}^\top (x_i\odot \phi_j) =0\}$, we should have 
$$\lambda_j {\bm{w}_j^*}-\sum_{i=1}^{n}\ell'_i y_i a_j^* \mathbb{I}\{{\bm{w}^*}^\top (x_i\odot \phi_j) >0\}{ (x_i\odot \phi_j)}\in \left\{\sum_{i\in I_j}\mu_{ji}\ell'_i y_i a_j^* { (x_i\odot \phi_j)}\Bigg| \mu_{ji}\in[0,1]\right\}.$$

\textbf{(3)} By result of part (2), we have for each $j\in[m]$, 
$$\lambda_j {\bm{w}_j^*}-\sum_{i=1}^{n}\ell'_i y_i a_j^* \mathbb{I}\{{\bm{w}_j^*}^\top (x_i\odot \phi_j) >0\}{ (x_i\odot \phi_j) }\in \left\{\sum_{i\in I_j}\mu_{ji}\ell'_i y_i a_j^* { (x_i\odot \phi_j) }\Bigg| \mu_{ji}\in[0,1]\right\}$$
where $\ell'_i\triangleq\ell'(-y_if(x_i;\bm{\theta}^*))$ and $I_j \triangleq\{i\in[n]:{\bm{w}^*_j}^\top (x_i\odot \phi_j) =0\}$. Since there is no inactive neuron at the local minimum $\bm{\theta}^*$, i.e., ${a_j^*}^2=\|\bm{w}_j^*\|_2^2>0$ for all $j\in[m]$. This indicates that for each $j\in[m]$, we have 
\begin{align*}
\frac{\lambda_j}{\sqrt{\|\bm{w}_j^*\|_2^2}} {\bm{w}_j^*}-\sgn(a_j^*)\sum_{i=1}^{n}&\ell'_i y_i  \mathbb{I}\{{\bm{w}_j^*}^\top (x_i\odot \phi_j) >0\}{ (x_i\odot \phi_j) }\in \left\{\sgn(a_j^*)\sum_{i\in I_j}\mu_{ji}\ell'_i y_i { (x_i\odot \phi_j) }\Bigg| \mu_{ji}\in[0,1]\right\}.
\end{align*}
We also note here that 
$$\left\{\sum_{i\in I_j}\mu_{ji}\ell'_i y_i a_j^* { (x_i\odot \phi_j) }\Bigg| \mu_{ji}\in[0,1]\right\}\subset \text{Span}\left(\left\{{ (x_i\odot \phi_j) }: i\in I_j\right\}\right)\triangleq \mathcal{S}({I_j}),$$
where we define $\mathcal{S}({I_j})$ as the linear span of the vectors in the set $\left\{{ (x_i\odot \phi_j) }: i\in I_j\right\}$. 
Furthermore, by the definition of $I_j$, we have the fact that for each $j\in[m]$, 
$${\bm{w}_j^*}\perp { (x_i\odot \phi_j) }, \quad\text{for }\forall i\in I_j.$$
which further indicates 
$$\frac{\lambda_j}{\sqrt{\|\bm{w}_j^*\|_2^2}} {\bm{w}_j^*}\perp \bm{z},\quad \text{for }\forall \bm{z}\in\mathcal{S}(I_j).$$
This indicates that 
\begin{align*}
-\frac{\lambda_j}{\sqrt{\|\bm{w}_j^*\|_2^2}} {\bm{w}_j^*}&+\sgn(a_j^*)\sum_{i=1}^{n}\ell'_i y_i  \mathbb{I}\{{\bm{w}_j^*}^\top (x_i\odot \phi_j) >0\}{ (x_i\odot \phi_j) }\\
&= \mathcal{P}_{\mathcal{S}(I_j)}\left(\sgn(a_j^*)\sum_{i=1}^{n}\ell'_i y_i  \mathbb{I}\{{\bm{w}_j^*}^\top (x_i\odot \phi_j) >0\}{ (x_i\odot \phi_j) }\right),
\end{align*}
where $\mathcal{P}_{\mathcal{S}(I_j)}(\bm{v})$ denotes the orthogonal projection of the vector $\bm{v}$ onto the vector space $\mathcal{S}(I_j)$. This indicates that for each $j\in[m]$,
\begin{align*}
\frac{\lambda_j}{\sqrt{\|\bm{w}_j^*\|_2^2}} {\bm{w}_j^*}&=\sgn(a_j^*)\sum_{i=1}^{n}\ell'_i y_i  \mathbb{I}\{{\bm{w}_j^*}^\top (x_i\odot \phi_j) >0\}{ (x_i\odot \phi_j) }\\
&- \mathcal{P}_{\mathcal{S}(I_j)}\left(\sgn(a_j^*)\sum_{i=1}^{n}\ell'_i y_i  \mathbb{I}\{{\bm{w}_j^*}^\top (x_i\odot \phi_j) >0\}{ (x_i\odot \phi_j)}\right)\\
&=\sgn(a_j^*)\sum_{i=1}^{n}\ell'_i y_i  \mathbb{I}\{{\bm{w}_j^*}^\top (x_i\odot \phi_j) >0\}\left[{ x_i\odot \phi_j }-\mathcal{P}_{\mathcal{S}(I_j)}\left({ x_i\odot \phi_j }\right)\right]
\end{align*}
where the last equation follows from the fact that the orthogonal projection $\mathcal{P}_{\mathcal{S}(I_j)}$ is a linear map. 
If we take the 2-norm on the both sides of the above equation, then we should have for each $j\in[m]$, 
\begin{equation*}
\lambda_j =\left\|\sum_{i=1}^{n}\ell'_i y_i  \mathbb{I}\{{\bm{w}_j^*}^\top (x_i\odot \phi_j) >0\}\left[{x_i\odot \phi_j }-\mathcal{P}_{\mathcal{S}(I_j)}\left({x_i\odot \phi_j }\right)\right]\right\|_2.
\end{equation*}
Given a matrix $\bm{C}=(c_{ij})\in\{0,1\}^{n\times m}$, a dataset $\mathcal{D}=\{(x_i, y_i)\}_{i=1}^n$ and an index set $I_j\subseteq[n]$, we define the following function for each $j\in[m]$, 
$$\Lambda_j(\bm{z}; \bm{C}, \mathcal{D},I_j)=\left\|\sum_{i=1}^{n}z_iy_ic_{ij}\left[{x_i\odot \phi_j }-\mathcal{P}_{\mathcal{S}(I_j)}\left({x_i\odot \phi_j }\right)\right]\right\|_2$$
where $\bm{z}=(z_1,...,z_n)\in\mathbb{R}^n$. Furthermore, we define
$$\bm{\Lambda}(\bm{z}; \bm{C}, \mathcal{D}, I_1,...,I_m)=(\Lambda_1(\bm{z}; \bm{C}, \mathcal{D}, I_1),...,\Lambda_m(\bm{z}; \bm{C}, \mathcal{D}, I_m)).$$
Now, it is easy to see that for a given index $j\in[m]$, matrix $\bm{C}\in\{0,1\}^{n\times m}$, a dataset $\mathcal{D}=\{(x_i, y_i)\}_{i=1}^n$ and a index set $I_j\subseteq[n]$, the function function $\bm{\Lambda}$  is a globally Lipschitz function by the fact that 
\begin{align*}
\|\bm{\Lambda}(\bm{z}; \bm{C}, \mathcal{D}, I_1,...,I_m)-\bm{\Lambda}&(\bm{z}'; \bm{C}, \mathcal{D}, I_1,...,I_m)\|_2\\
&\le\sqrt{m}\max_{j\in[m]}\left|\Lambda_j(\bm{z}; \bm{C}, \mathcal{D}, I_j)-\Lambda_j(\bm{z}'; \bm{C}, \mathcal{D}, I_j)\right|
\end{align*}
and the fact that for any $j\in[m]$
\begin{align*}
\left|\Lambda_j(\bm{z}; \bm{C}, \mathcal{D}, I_j)-\Lambda_j(\bm{z}'; \bm{C}, \mathcal{D}, I_j)\right|&\le\left\|\sum_{i=1}^{n}(z_i-z'_i)y_ic_{ij}\left[{x_i\odot \phi_j}-\mathcal{P}_{\mathcal{S}(I_j)}\left({x_i\odot \phi_j}\right)\right]\right\|_2\\
&\le(1+D)\|\bm{z}-\bm{z}'\|_2
\end{align*}
where $D\triangleq\max_i\|x_i\|_2$. Therefore, 
\begin{align*}
(\lambda_1,...,\lambda_m)\in\bigcup_{\bm{C}\in\{0,1\}^{n\times m}}\bigcup_{I_1\subseteq[m]}...\bigcup_{I_n\subseteq[m]}\bm{\Lambda}(\mathbb{R}^n; \bm{C}, \mathcal{D}, I_1,...,I_m).
\end{align*}

\textbf{(4)} If $$((\lambda_1,...,\lambda_m)\notin\bigcup_{\bm{C}\in\{0,1\}^{n\times m}}\bigcup_{I_1\subseteq[m]}...\bigcup_{I_n\subseteq[m]}\bm{\Lambda}(\mathbb{R}^n; \bm{C}, \mathcal{D}, I_1,...,I_m))\triangleq \mathcal{C}(\mathcal{D}),$$
then by the result of the  part (3), we have there is an inactive neuron. 

\end{proof}

\subsection{Proof  of Theorem~\ref{thm::single}}\label{appendix::proof::thm-single}
\begin{proof}
(1) 
Let the minimal regularization coefficient
be $ \lambda_{\min} \triangleq \min_{j\in[m]} \{ \lambda_j \} > 0 $.
Since the univariate loss $\ell$ is non-negative, therefore we can obtain
$
L_n(  \bm{\theta}; \bm{\lambda}) \ge \lambda_{\min}   \| \bm{\theta}  \|_{2}^2.
$
Thus, the empirical loss $\lossc$ is coercive since $L_n(  \bm{\theta}; \bm{\lambda})\rightarrow \infty $  as $\| \bm{\theta} \|_{2} \rightarrow \infty$.

(2) 
By Lemma~\ref{lemma::single-1}, there exists a zero measure set $\mathcal{C}$ such that for any $\bm{\lambda}\notin \mathcal{C}$, at every local minimum $\bm{\theta}^{*}$ of the empirical loss $\lossc$, the neural network $f(x;\bm{\theta}^{*})$ has an inactive neuron. Without loss of generality, we assume that  $a_{1}^{*}=0, \|\bm{w}_{1}^{*}\|_{2}=0$ and $b_{1}^{*}=0$. Since $\bm{\theta}^{*}$ is a local minimum of the empirical loss $\lossc$, then there exists a $\delta>0$, such that for any $\tilde{\bm{\theta}}:\|\tilde{\bm{\theta}}-\bm{\theta}^{*}\|<2\delta$, $L(\tilde{\bm{\theta}};\bm{\lambda})\ge L(\bm{\theta}^{*};\bm{\lambda})$. Now we choose the perturbation where we only perturb the parameters of that inactive neuron, i.e., ${a_{1},\bm{w}_{1}}$. Let $\tilde{a}_{1}=\delta\sgn(\tilde{a}_{1})$, $\tilde{\bm{w}}_{1}=\bm{u}$ for arbitrary $\bm{u}:\|\bm{u}\|_{2}=1$ and $(\tilde{a_{j}},\tilde{\bm{w}}_{j})=(a_{j}^{*},\bm{w}^{*}_{j})$ for $j\neq 1$. By the second order Taylor expansion and the definition of the local minimum, we obtain that, for any $\sgn(\tilde{a}_{j})\in\{-1, 1\}$ and any unit  vector $\bm{u}:\|\bm{u}\|_{2}=1$,
\begin{align*}
L(\tilde{\bm{\theta}};\bm{\lambda})&=\loss-\sum_{i=1}^{n}\ell'_{i}y_{i}\delta^{2}\sgn(\tilde{a}_{1})(\bm{u}^{\top}(x_i\odot \phi_j))_{+}+R(\delta, \mathcal{D},\bm{u})\delta^{4}+\lambda_{1}\delta^{2}\\
&\ge \loss,
\end{align*}
where $R(\delta,\mathcal{D},\bm{u},v)$ is the second order remaining term in the Taylor expansion and $\ell'_i$ is the shorthanded notation for $\ell'(-y_{i}f(x_{i};\bm{\theta}^{*}))$.
This further indicates that for any $\bm{u}:\|\bm{u}\|_{2}=1$ and for any $j=1,...,d-r+1$,
$$\left|\sum_{i=1}^{n}\ell'(-y_{i}f(x_{i};\bm{\theta}^{*}))(-y_{i})(\bm{u}^{\top}(x_i\odot \phi_j))_{+}\right|\le \lambda_{1}<\lambda_{0}$$
By Theorem~\ref{thm::other-memo}, we should have 
$$\sum_{i=1}^{n}\ell'(-y_{i}f(x_{i};\bm{\theta}^{*}))\le (\lambda_0+2E)/\gamma.$$
Furthermore, by the convexity of the loss $\ell$, we should have 
$\ell'(z)\ge \ell'(0)\mathbb{I}\{z\ge 0\}$, which further indicates 
$$R_n(f;\bm{\theta}^*)=\frac{1}{n}\sum_{i=1}^n\mathbb{I}\{y_i\neq \sgn(f(x_i;\bm{\theta}^*))\}\le \frac{1}{\ell'(0)n}\sum_{i=1}^{n}\ell'(-y_{i}f(x_{i};\bm{\theta}^{*}))\le \frac{\lambda_0+2E}{ \ell'(0)\gamma n}.$$
\end{proof}

\section{Proof of Theorem~\ref{thm::single-gen}}\label{appendix::proof::gen}
\begin{proof}
Based on the results in Lemma~\ref{lemma::single-1} and proof of Theorem~\ref{thm::single} in Section~\ref{appendix::proof::thm-single}, we know that at every local minimum, the following conditions holds
\begin{itemize}
\item[(1)] $\|\bm{w}_j^*\|_2= |a_j^*| $ holds for any $j=1,...,m$, 
\item[(2)] $\|\nabla_{\bm{a}}L_n(\bm{\theta}^*)\|_2=0$,
\item[(3)]  for any $k=1,..., d-r+1$
$$\max_{\bm{u}\in\mathbb{B}^d}\left|\sum_{i=1}^n\ell'(-y_i f(x_i;\bm{\theta}^*))y_i(\bm{u}^\top (x_i\odot \phi_k))_+\right|\le \lambda_0,$$
\end{itemize}
Thus, by Theorem~\ref{thm::other-gen}, with probability at least $1-\delta$, at every local minimum $\bm{\theta}^*$, the following inequality holds, 
\begin{align*}
\mathbb{P}(Yf(X;\bm{\theta}^*)<0)
&\le \frac{\lambda_0+2E}{\gamma n\ell'(0)}+\frac{(6\lambda_0+12E)\ln n}{a\gamma \lambda_0}\cdot \frac{4\sqrt{2}}{\ell'(0)\sqrt{n}}+\frac{1}{\ell'(0)}\sqrt{\frac{\log(1/\delta)}{2n}}\\
&\le \frac{\lambda_0}{\gamma n\ell'(0)}+ \frac{2E}{\gamma n\ell'(0)}+\frac{24\sqrt{2}\ln n}{a\gamma\ell'(0)\sqrt{n}}+\frac{48\sqrt{2}E\ln n}{a\gamma \ell'(0)\lambda_0 \sqrt{n}}+\frac{1}{\ell'(0)}\sqrt{\frac{\log(1/\delta)}{2n}}\\
&=\frac{\lambda_0}{\gamma n\ell'(0)}+ \frac{48\sqrt{2}E\ln n}{a\gamma \ell'(0)\lambda_0 \sqrt{n}}+\frac{2E}{\gamma n\ell'(0)}+\frac{24\sqrt{2}\ln n}{a\gamma\ell'(0)\sqrt{n}}+\frac{1}{\ell'(0)}\sqrt{\frac{\log(1/\delta)}{2n}}\\
&=\mathcal{O}\left(\frac{\lambda_0}{\gamma n}+\frac{E\ln n}{\gamma\lambda_0\sqrt{n}}+\frac{E}{\gamma n}+\frac{\ln n}{\gamma\sqrt{n}}+\sqrt{\frac{\log(1/\delta)}{n}}\right)
\end{align*}
\end{proof}

\section{Gradient is Locally Lipschitz}
For the loss function 
$$L_n(\bm{\theta})=\sum_{i=1}^n\ell(-y_i f(x_i;\bm{\theta}))+\sum_{j=1}^m\lambda_ja_j^2 $$
where the neural network is defined as 
$$f(x;\bm{\theta})=\sum_{j=1}^m a_j(a_j\bm{u}_j^\top (x_i\odot \phi_j))_+,$$
and $\|\bm{u}_j\|_2=1$ for all $j=1,...,m$. We note here that since during the gradient descent, the direction vectors $\bm{u}_j$'s are fixed. Now the parameter vector $\bm{\theta}$ becomes $\bm{\theta}=(a_1,...,a_m, a_1\bm{u}_1,...,a_m\bm{u}_m)$. 
Now the gradient is written as 
\begin{align*}
\nabla_{a_j}L_n(\bm{\theta})&=2\sum_{i=1}^n\ell'(-y_i f(x_i;\bm{\theta}))(-y_i)(a_j\bm{u}_j^\top (x_i\odot \phi_j))_++2\lambda_ja_j
\end{align*}

\begin{lemma}
For any $\tilde{\bm{a}}=(\tilde{a}_1,...,\tilde{a}_m)$ in a small local region of $\bm{\theta}$: $$|\tilde{a}_j-a_j|\le |a_j|/2,$$
the gradient is locally Lipschitz  with a Lipschitz constant depending on $\|\bm{a}\|_2$:
$$\|\nabla_{\bm{a}}L_n(\tilde{\bm{\theta}})-\nabla_{\bm{a}}L_n(\bm{\theta})\|_2\le  12\|\bm{a}-\tilde{\bm{a}}\|_2\sqrt{n^2\|\bm{a}\|_2^4+n^2+\lambda_0^2}.$$
\end{lemma}
\begin{proof}
Since 
\begin{align*}
&|\nabla_{a_j}L_n(\bm{\theta})-\nabla_{a_j}L_n(\tilde{\bm{\theta}})|\\
&\le2\left|\sum_{i=1}^n\ell'(-y_i f(x_i;\bm{\theta}))(-y_i)(a_j\bm{w}_j^\top (x_i\odot \phi_j))_+-\sum_{i=1}^n\ell'(-y_i f(x_i;\tilde{\bm{\theta}}))(-y_i)(\tilde{a}_j\bm{w}_j^\top (x_i\odot \phi_j))_+\right|\\
&\quad+2\lambda_j|a_j-\tilde{a}_j|\\
&\le 2\sum_{i=1}^n\left|\ell'(-y_i f(x_i;\bm{\theta}))(a_j\bm{w}_j^\top (x_i\odot \phi_j))_+-\ell'(-y_i f(x_i;\tilde{\bm{\theta}}))(\tilde{a}_j\bm{w}_j^\top (x_i\odot \phi_j))_+\right|+2\lambda_j|a_j-\tilde{a}_j|\\
&\le 2\sum_{i=1}^n\left(|\ell'(-y_i f(x_i;\bm{\theta}))-\ell'(-y_i f(x_i;\tilde{\bm{\theta}}))||a_j|+\ell'(-y_i f(x_i;\tilde{\bm{\theta}}))|a_j-\tilde{a}_j|\right)+2\lambda_j|a_j-\tilde{a}_j|\\
&\le 2\sum_{i=1}^n|\ell'(-y_i f(x_i;\bm{\theta}))-\ell'(-y_i f(x_i;\tilde{\bm{\theta}}))||a_j|+{n|a_j-\tilde{a}_j|}+2\lambda_j|a_j-\tilde{a}_j|
\end{align*}
and using the fact that $\bm{w}_j=a_j\bm{u}_j$ to get 
\begin{align*}
&|\ell'(-y_i f(x_i;\bm{\theta}))-\ell'(-y_i f(x_i;\tilde{\bm{\theta}}))|\le {|f(x_i;\bm{\theta})-f(x_i;\tilde{\bm{\theta}})|}\\
&\le \sum_{i=1}^m|a_j(\bm{w}_j^\top (x_i\odot \phi_j))_+-\tilde{a}_j(\tilde{\bm{w}}_j^\top (x_i\odot \phi_j))_+|\le \sum_{j=1}^m \left(|a_j-\tilde{a}_j|\|\bm{w}_j\|_2+|\tilde{a}_j|\|\bm{w}_j-\tilde{\bm{w}}_j\|_2\right)\\
&\le \left(\sum_{j=1}^m |a_j-\tilde{a}_j||a_j|+\frac{3}{2}|{a}_j|\|a_j\bm{u}_j-\tilde{a}_j\bm{u}_j\|_2\right)\\
&\le 3\sum_{j=1}^m|a_j||\tilde{a}_j-a_j|\le 3\|\bm{a}\|_2\|\tilde{\bm{a}}-\bm{a}\|_2
\end{align*}
Thus, we have 
\begin{align*}
|\nabla_{a_j}L_n(\bm{\theta})-\nabla_{a_j}L_n(\tilde{\bm{\theta}})|&\le 6n\|\bm{a}\|_2\|\tilde{\bm{a}}-\bm{a}\|_2|a_j|+{n|a_j-\tilde{a}_j|}+2\lambda_j|a_j-\tilde{a}_j|
\end{align*}
and 
\begin{align*}
\|\nabla_{\bm{a}}L_n(\bm{\theta})-\nabla_{\bm{a}}L_n(\tilde{\bm{\theta}})\|_2&\le \sqrt{108n^2\|\bm{a}\|^4_2\|\tilde{\bm{a}}-\bm{a}\|_2^2+{3(n^2+4\lambda_j^2)\|\bm{a}-\tilde{\bm{a}}\|_2^2}}\\
&\le 12\|\bm{a}-\tilde{\bm{a}}\|_2\sqrt{n^2\|\bm{a}\|_2^4+n^2+\lambda_0^2}
\end{align*}

\end{proof}

\section{Perturbing The Inactive Neuron}

Recall that the empirical loss is defined as 
\begin{align*}
L_n(\bm{\theta})=\sum_{i=1}^n\ell(-y_i f(x_i;\bm{\theta}))+\frac{1}{2}\sum_{j=1}^m\lambda_j\left(a_j^2 +\|\bm{w}_j\|^2_2\right)
\end{align*}
where $\bm{\theta}=(a_1,...,a_m, a_1\bm{u}_1,...,a_m\bm{u}_m )$ where $\bm{u}_j$'s are all unit vectors, i.e., $\|\bm{u}\|_2=1$. 

\begin{lemma}\label{lemma::perturbation}
If at the point $\bm{\theta}=(a_1,...,a_m, a_1\bm{u}_1,...,a_m\bm{u}_m)$,  $|a_{j_0}|\le \frac{1}{\sqrt{4n}}$ for some index $j_0\in[m]$ and $\bm{\lambda}\in[0, \lambda_0]^m$ with $\lambda_0\ge \sqrt{n}$,
then by setting 
\begin{align*}
\tilde{a}_{j_0}&=\sgn\left(\sum_{i=1}^n\ell'(-y_i f(x_i;\bm{\theta}))y_i (\bm{v}^\top (x_i\odot \phi_j))_+\right)\sqrt{\frac{1}{\lambda_0}},\\
\tilde{\bm{u}}_{j_0}&=\bm{v}\sgn\left(\sum_{i=1}^n\ell'(-y_i f(x_i;\bm{\theta}))y_i (\bm{v}^\top (x_i\odot \phi_j))_+\right)
\end{align*}
and setting $\tilde{a}_j=a_j$ and $\tilde{\bm{u}}_j=\bm{u}_j$ for any $j\neq j_0$, 
we have 
 $$L_n(\tilde{\bm{\theta}};\bm{\lambda})-L_n(\bm{\theta};\bm{\lambda})\le -\frac{1}{\lambda_0}\left|\sum_{i=1}^n \ell'(-y_i f(x_i;\bm{\theta}))y_i (\bm{v}_{j_0}^\top (x_i\odot \phi_{j_0}))_+\right|+4.$$
\end{lemma}

\begin{proof}
Let $\delta =1/\sqrt{4n}$. For simplicity of notation, we assume that $|a_1|\le\delta$ and $\|a_1 \bm{u}_1\|_2=|a_1|<\delta$. Now we  perturb the parameter vector $\bm{\theta}$ to $\tilde{\bm{\theta}}$ in the following way: perturbing $(a_1, a_1\bm{u}_1)$ to $(\tilde{a}_1, \tilde{a}_1\tilde{\bm{u}}_1)$ with $\|\tilde{\bm{u}}_1\|=1$ and $|\tilde{a}_1|=\frac{1}{\sqrt{\lambda_0}}$ and setting $\tilde{a}_j=a_j$ and $\tilde{\bm{u}}_j=\bm{u}_j$ for any $j\neq 1$. Now we see how it improves the empirical loss. By definition, we have 
\begin{align*}
L_n(\tilde{\bm{\theta}};\bm{\lambda})-L_n(\bm{\theta};\bm{\lambda})&=\sum_{i=1}^n\left[\ell(-y_i f(x_i;\tilde{\bm{\theta}}))-\ell(-y_i f(x_i;{\bm{\theta}}))\right]+\frac{1}{2}\sum_{j=1}\lambda_j(\tilde{a}_j^2+\|\tilde{a}_j\tilde{\bm{u}}_j\|^2_2-{a}_j^2-\|{a}_j{\bm{u}}_j\|^2_2).
\end{align*}
Since the  loss is strongly smooth, i.e., for any $z, z'\in\mathbb{R}$:
$$\ell(z')-\ell(z)\le \ell'(z)(z'-z)+\frac{1}{2}|z'-z|^2,$$
then we should have 
\begin{align*}
&L_n(\tilde{\bm{\theta}};\bm{\lambda})-L_n(\bm{\theta};\bm{\lambda})\\
&=\sum_{i=1}^n\left[\ell(-y_i f(x_i;\tilde{\bm{\theta}}))-\ell(-y_i f(x_i;{\bm{\theta}}))\right]+\frac{1}{2}\sum_{j=1}\lambda_j(\tilde{a}_j^2+\|\tilde{a}_j\tilde{\bm{u}}_j\|^2_2-{a}_j^2-\|{a}_j{\bm{u}}_j\|^2_2)\\
&\le \sum_{i=1}^n\left[\ell'(-y_if(x_i;\bm{\theta}))(-y_i)(f(x_i;\tilde{\bm{\theta}})-f(x_i;{\bm{\theta}}))+\frac{1}{2}|f(x_i;\tilde{\bm{\theta}})-f(x_i;{\bm{\theta}})|^2\right]+{\lambda_1(\tilde{a}_1^2-{a}_1^2)}\\
&=\sum_{i=1}^n\ell'(-y_if(x_i;\bm{\theta}))(-y_i)[\tilde{a}_1(\tilde{a}_1\tilde{\bm{u}}_1^\top (x_i\odot \phi_1))_+-{a}_1({a}_1{\bm{u}}_1^\top (x_i\odot \phi_1))_+]\\
&\quad+\frac{1}{2}\sum_{i=1}^n|\tilde{a}_1(\tilde{a}_1\tilde{\bm{u}}_1^\top (x_i\odot \phi_1))_+-{a}_1({a}_1{\bm{u}}_1^\top (x_i\odot \phi_1))_+|^2+{\lambda_1(\tilde{a}_1^2-{a}_1^2)}\\
&=-\sum_{i=1}^n\ell'(-y_if(x_i;\bm{\theta}))y_i\tilde{a}_1(\tilde{a}_1\tilde{\bm{u}}_1^\top (x_i\odot \phi_1))_++{\lambda_1\tilde{a}_1^2}&&\triangleq I_1\\
&\quad+\sum_{i=1}^n\ell'(-y_if(x_i;\bm{\theta}))y_i{a}_1({a}_1{\bm{u}}_1^\top (x_i\odot \phi_1))_+-\lambda_1{a}_1^2&&\triangleq I_2\\
&\quad +\frac{1}{2}\sum_{i=1}^n|\tilde{a}_1(\tilde{a}_1\tilde{\bm{u}}_1^\top (x_i\odot \phi_1))_+-{a}_1({a}_1{\bm{u}}_1^\top (x_i\odot \phi_1))_+|^2&&\triangleq I_3\\
\end{align*}
Now we let 
\begin{align*}
\tilde{a}_1&=\xi\sgn\left(\sum_{i=1}^n\ell'(-y_i f(x_i;\bm{\theta}))y_i (\bm{v}^\top (x_i\odot \phi_1))_+\right),\\
\tilde{\bm{u}}_1&=\bm{v}\sgn\left(\sum_{i=1}^n\ell'(-y_i f(x_i;\bm{\theta}))y_i (\bm{v}^\top (x_i\odot \phi_1))_+\right).
\end{align*}
Therefore, we have 
\begin{align*}
I_1&\triangleq -\sum_{i=1}^n\ell'(-y_if(x_i;\bm{\theta}))y_i\tilde{a}_1(\tilde{a}_1\tilde{\bm{u}}_1^\top (x_i\odot \phi_1))_++{\lambda_1\tilde{a}_1^2}\\
&=-\frac{1}{\lambda_0}\left|\sum_{i=1}^n\ell'(-y_i f(x_i;\bm{\theta}))y_i (\bm{v}^\top (x_i\odot \phi_1))_+\right|+\lambda_1/\lambda_0\\
&\le -\frac{1}{\lambda_0}\left|\sum_{i=1}^n\ell'(-y_i f(x_i;\bm{\theta}))y_i (\bm{v}^\top (x_i\odot \phi_1))_+\right|+1,
\end{align*}
\begin{align*}
I_2&\triangleq \sum_{i=1}^n\ell'(-y_if(x_i;\bm{\theta}))y_i{a}_1({a}_1{\bm{u}}_1^\top (x_i\odot \phi_j))_+-{\lambda_1{a}_1^2}\le \sum_{i=1}^n\ell'(-y_if(x_i;\bm{\theta}))y_i{a}_1({a}_1{\bm{u}}_1^\top (x_i\odot \phi_j))_+\\
&\le \sum_{i=1}^n\ell'(-y_if(x_i;\bm{\theta}))\left|y_i{a}_1({a}_1{\bm{u}}_1^\top (x_i\odot \phi_j))_+\right|\le \delta^2\sum_{i=1}^n\ell'(-y_if(x_i;\bm{\theta}))\le n\delta^2=\frac{1}{4},
\end{align*}
and
\begin{align*}
I_3&\triangleq  \frac{1}{2}\sum_{i=1}^n|\tilde{a}_1(\tilde{a}_1\tilde{\bm{u}}_1^\top (x_i\odot \phi_j))_+-{a}_1({a}_1{\bm{u}}_1^\top (x_i\odot \phi_j))_+|^2\\
&\le \sum_{i=1}^n\left[|\tilde{a}_1(\tilde{a}_1\tilde{\bm{u}}_1^\top (x_i\odot \phi_j))_+|^2+|{a}_1({a}_1{\bm{u}}_1^\top (x_i\odot \phi_j))_+|^2\right]\\
&\le \sum_{i=1}^n [\delta^4+\lambda_0^{-2}]={n(\lambda_0^{-2}+\delta^4)}\le 2.
\end{align*}
This indicates that 
\begin{align*}
L_n(\tilde{\bm{\theta}};\bm{\theta})-L_n(\bm{\theta};\bm{\theta})&\le I_1+I_2+I_3\\
&\le -\frac{1}{\lambda_0}\left|\sum_{i=1}^n\ell'(-y_i f(x_i;\bm{\theta}))y_i (\bm{v}^\top (x_i\odot \phi_1))_+\right|+1+2+\frac{1}{4}\\
&\le -\frac{1}{\lambda_0}\left|\sum_{i=1}^n\ell'(-y_i f(x_i;\bm{\theta}))y_i (\bm{v}^\top (x_i\odot \phi_1))_+\right|+4
\end{align*}

\end{proof}

\section{Choice of Regularization Coefficient}\label{appendix::choose-ceof}
In this section, we show how to solve the following problem and we get rid of the index $k$ for simplicity of notations.  

\begin{problem}~\label{problem::1}
For a given interger $K>0$, a  series $\{\phi_j\}$ of period equal to $d-r+1$, a scalar series $\{a_j\}_j$, a unit vector series  $\{\bm{u}_j\}_j$ and a vector $\bm{\lambda}_{\text{old}}$, how to find a vector $\bm{\lambda}$ such that for each $s\in[d-r+1]$, all of the following three inequalities 
\begin{equation}\label{eq::coefficient::2}
\min_{p_1,...,p_n\in\mathbb{R}}\sum_{j:\phi_j=\phi_s}\left(\lambda_j-\sgn(a_j)\sum_{i=1}^np_iy_i(\sgn(a_j)\bm{u}_j^\top (x_i\odot \phi_j))_+\right)^2\ge \left(\frac{\lambda_0}{8K}\right)^2
\end{equation}
\begin{equation}\label{eq::coefficient::3}
\bm{\lambda}\preccurlyeq \bm{\lambda}_{\text{old}}
\end{equation}
and 
\begin{equation}\label{eq::coefficient::4}
\|\bm{\lambda}-\bm{\lambda}_\text{old}\|_\infty\le \frac{\lambda_0}{2K}
\end{equation}
holds?
\end{problem} 

\textbf{Notation.} To further elaborate this problem, we define the following notation. Given a periodic series $\{\phi_j\}$ of period equal to $d-r+1$, for any vector $\bm{q}=(q_j)_j$ and for each $s\in[d-r+1]$, we define a cropped vector $\bm{q}[s]$ of $\bm{q}$ where $\bm{q}[s]\triangleq (q_j)_{j:\phi_j=\phi_s}$.  

Therefore, if we define the vector 
\begin{equation*}
\bm{q}_{i}=\left(y_i\sgn(a_j)(\sgn(a_j)\bm{u}_j^\top (x_i\odot \phi_j))_+ \right)_{j},
\end{equation*}
 then for each $s\in[d-r+1]$, the left hand side of the inequality~\eqref{eq::coefficient::2} becomes 
\begin{align*}
&\min_{p_1,...,p_n\in\mathbb{R}}\sum_{j:\phi_j=\phi_s}\left(\lambda_j-\sgn(a_j)\sum_{i=1}^np_iy_i(\sgn(a_j)\bm{u}_j^\top (x_i\odot \phi_s))_+\right)^2=\min_{p_1,...,p_n\in\mathbb{R}}\left\|\bm{\lambda}[s]-\sum_{i=1}^np_i\bm{q}_{i}[s]\right\|_2^2.
\end{align*}
This indicates that to solve the Problem~\ref{problem::1}, we only need to solve  the following problem.
\begin{problem}~\label{problem::2}
For a given integer $K>0$, a  series $\{\phi_j\}$ of period equal to $d-r+1$, a series $\{\bm{q}_i\}_i$ of vectors of dimension $m$ and a vector $\bm{\lambda}_{\text{old}}$, how to find a vector $\bm{\lambda}$ such that for each $s\in[d-r+1]$, all of the following three inequalities 
\begin{equation}\label{eq::coefficient::5}
\min_{p_1,...,p_n\in\mathbb{R}}\left\|\bm{\lambda}[s]-\sum_{i=1}^np_i\bm{q}_{i}[s]\right\|_2^2\ge \left(\frac{\lambda_0}{8K}\right)^2
\end{equation}
\begin{equation}\label{eq::coefficient::6}
\bm{\lambda}[s]\preccurlyeq \bm{\lambda}_{\text{old}}[s]
\end{equation}
and 
\begin{equation}\label{eq::coefficient::7}
\|\bm{\lambda}[s]-\bm{\lambda}_\text{old}[s]\|_\infty\le \frac{\lambda_0}{2K}
\end{equation}
holds?
\end{problem} 
We note here that the dimension of  $\bm{q}_i[s]$ for each $s$ is no less than $n+1$ since the series $\{\phi_j\}$ has a period of $d-r+1$ and the dimension of $\bm{q}_i$ is $m\ge(d-r+1)(n+1)$. Furthermore, we can ignore the index $s$ since $\bm{\lambda}[1],...,\bm{\lambda}[d-r+1]$ is a partition of the vector $\bm{\lambda}$ and we can focus the sub-problem for each $s\in[d-r+1]$. Therefore, to solve the Problem~\ref{problem::2}, we only need to solve the following problem. 
\begin{problem}~\label{problem::3}
For a given integer $K>0$, a  series $\{\bm{q}_i\}_i$ of vectors of dimension no less than $n+1$ and a vector $\bm{\rho}_{\text{old}}$, how to find a vector $\bm{\rho}$ such that  all of the following three inequalities 
\begin{equation}
\min_{p_1,...,p_n\in\mathbb{R}}\left\|\bm{\rho}-\sum_{i=1}^np_i\bm{q}_{i}\right\|_2^2\ge \left(\frac{\lambda_0}{8K}\right)^2
\end{equation}
\begin{equation}
\bm{\rho}\preccurlyeq \bm{\rho}_{\text{old}}
\end{equation}
and 
\begin{equation}
\|\bm{\rho}-\bm{\rho}_\text{old}\|_\infty\le \frac{\lambda_0}{2K}
\end{equation}
holds?
\end{problem} 
Now we present the algorithm solving this problem.

\subsection{Algorithm and Analysis} 
We assume that each vector $\bm{q}_i$ is of dimension $Q\ge n+1$. We first find a unit vector $\bm{v}\in\mathbb{R}^Q:\|\bm{v}\|_2=1$ such that $\bm{v}^\top \bm{q}_i=0$ holds for all $i\in[n]$. We note here that finding such a non-zero vector is always possible since $Q>n$. This indicates that $\text{Span}\{\bm{q}_1,...,\bm{q}_n\}\subseteq \{\bm{q}\in\mathbb{R}^m:\bm{v}^\top \bm{q}=0\}$. we have 
\begin{align*}\min_{p_1,...,p_n\in\mathbb{R}}\left\|\bm{\rho}-\sum_{i=1}^np_i\bm{q}_i\right\|_2&= d(\bm{\rho}, \text{Span}\{\bm{q}_1,...,\bm{q}_n\}) \\
&\ge d(\bm{\rho}, \{\bm{q}\in\mathbb{R}^m:\bm{v}^\top \bm{q}=0\})\\
&=|\bm{v}^\top \bm{\rho}|
\end{align*}
where $d({\rho}, V)=\inf_{x\in V}\|x-\rho\|_2$. Now we show how to choose the coefficient vector $\bm{\rho}$ satisfying inequalities~\eqref{eq::coefficient::5}, \eqref{eq::coefficient::6} and \eqref{eq::coefficient::7}.

\textbf{Case 1:} If $|\bm{v}^\top\bm{\rho}_\text{old}|\ge \frac{\rho_0}{8K}$, then we set $\bm{\rho}=\bm{\rho}_\text{old}$ and we also have 
$$\bm{\rho}\preccurlyeq \bm{\rho}_\text{old},$$
$$\|\bm{\rho}-\bm{\rho}_\text{old}\|_\infty=0\le \frac{\rho_0}{2K},$$
and 
\begin{align*}
\min_{p_1,...,p_n\in\mathbb{R}}\left\|\bm{\rho}-\sum_{i=1}^np_i\bm{q}_i\right\|_2\ge |\bm{v}^\top \bm{\rho}|=|\bm{v}^\top \bm{\rho}_\text{old}|\ge \frac{\lambda_0}{8K}.
\end{align*}

\textbf{Case 2:} If $|\bm{v}^\top\bm{\rho}_\text{old}|< \frac{\lambda_0}{8K}$, then we update $\bm{\rho}$ in the following way.   Since we have $$1=\sum_{j=1}^Qv_j^2=\sum_{j=1}^Qv_j^2\mathbb{I}\{v_j>0\}+\sum_{j=1}^Qv_j^2\mathbb{I}\{v_j<0\}$$
which indicates that 
$$\sum_{j=1}^Qv_j^2\mathbb{I}\{v_j>0\}\ge 1/2 \quad\text{or}\quad \sum_{j=1}^Qv_j^2\mathbb{I}\{v_j<0\}\ge 1/2.$$
(1) If $\sum_{j=1}^Qv_j^2\mathbb{I}\{v_j>0\}\ge 1/2$, then define 
$$\Delta \bm{\rho}=\frac{\lambda_0}{2K}\left(\begin{matrix}v_1\mathbb{I}\{v_1>0\}\\.\\.\\.\\ v_Q\mathbb{I}\{v_Q>0\}\end{matrix}\right)$$
and update 
$$\bm{\rho}=\bm{\rho}_\text{old}-\Delta \bm{\rho}.$$
Therefore,
$$\bm{\rho}\preccurlyeq \bm{\rho}_\text{old},$$
$$\|\bm{\rho}-\bm{\rho}_\text{old}\|_\infty\le \|\bm{\rho}-\bm{\rho}_\text{old}\|_2=\|\Delta \bm{\rho}\|_2 \le \frac{\lambda_0}{2K},$$
and 
\begin{align*}
 \min_{p_1,...,p_n\in\mathbb{R}}\left\|\bm{\rho}-\sum_{i=1}^np_i\bm{q}_i\right\|_2&\ge |\bm{v}^\top \bm{\rho}|=\left|\bm{v}^\top [\bm{\rho}-\Delta \bm{\rho}]\right|\\
&\ge|\bm{v}^\top\Delta \bm{\rho}|-\left|\bm{v}^\top \bm{\rho}\right|\\
&\ge \frac{\lambda_0}{2K}\sum_{j=1}^mv_j^2\mathbb{I}\{v_j>0\}-\frac{{\lambda_0}}{8K}\\
&\ge \frac{\lambda_0}{4K}-\frac{{\lambda_0}}{8K}=\frac{\lambda_0}{8K}
\end{align*}

(2) If $\sum_{j=1}^Qv_j^2\mathbb{I}\{v_j<0\}\ge 1/2$, then define 
$$\Delta \bm{\rho}=\frac{\lambda_0}{2K}\left(\begin{matrix}v_1\mathbb{I}\{v_1<0\}\\.\\.\\.\\ v_m\mathbb{I}\{v_m<0\}\end{matrix}\right)$$
and update 
$$\bm{\rho}=\bm{\rho}_\text{old}+\Delta \bm{\rho}.$$
Therefore,
$$\bm{\rho}\preccurlyeq \bm{\rho}_\text{old},$$
$$\|\bm{\rho}-\bm{\rho}_\text{old}\|_\infty\le \|\bm{\rho}-\bm{\rho}_\text{old}\|_2=\|\Delta \bm{\rho}\|_2 \le \frac{\lambda_0}{2K},$$
and 
\begin{align*}
\min_{p_1,...,p_n\in\mathbb{R}}\left\|\bm{\rho}-\sum_{i=1}^np_i\bm{q}_i\right\|_2&\ge |\bm{v}^\top \bm{\rho}|=\left|\bm{v}^\top [\bm{\rho}+\Delta \bm{\rho}]\right|\\
&\ge|\bm{v}^\top\Delta \bm{\rho}|-\left|\bm{v}^\top \bm{\rho}\right|\\
&\ge \frac{\lambda_0}{2K}\sum_{j=1}^mv_j^2\mathbb{I}\{v_j<0\}-\frac{{\lambda_0}}{8K}\\
&\ge \frac{\lambda_0}{4K}-\frac{{\lambda_0}}{8K}=\frac{{\lambda_0}}{8K}
\end{align*}

Above all,  we have 
$$\bm{\rho}\preccurlyeq \bm{\rho}_\text{old},$$
$$\|\bm{\rho}-\bm{\rho}_\text{old}\|_\infty\le \|\bm{\rho}-\bm{\rho}_\text{old}\|_2=\|\Delta \bm{\rho}\|_2 \le \frac{\lambda_0}{2K},$$
and 
\begin{align*}
\min_{p_1,...,p_n\in\mathbb{R}}\left\|\bm{\rho}-\sum_{i=1}^np_i\bm{q}_i\right\|_2\ge\frac{{\rho_0}}{8K}.
\end{align*}

\textbf{Computational Complexity.} Now we compute the complexity of this algorithm. To find the unit vector $\bm{v}$ such that $\bm{v}^\top \bm{q}_i=0$ for each $i\in[n]$, we only need to find the SVD of the matrix $[\bm{q}_1,...,\bm{q}_n]$ which  has a complexity of $\mathcal{m^2}$. Therefore, solving  problem~\ref{problem::1} has a complexity of $\mathcal{O}(dm^2)$

\section{Analysis on The Gradient Descent}
For the loss function 
$$L_n(\bm{\theta})=\sum_{i=1}^n\ell(-y_i f(x_i;\bm{\theta}))+\sum_{j=1}^m\lambda_ja_j^2$$
where the neural network is defined as 
$f(x;\bm{\theta})=\sum_{j=1}^m a_j(a_j\bm{u}_j^\top (x_i\odot \phi_j))_+$
and $\|\bm{u}_j\|_2=1$ for all $j=1,...,m$. We note here that since during the gradient descent, the direction vectors $\bm{u}_j$'s are fixed, then the parameter vector of the neural network is $\bm{\theta}=(a_1,...,a_m,a_1\bm{u}_1,...,a_m\bm{u}_m)$. 
Now the gradient is  
\begin{align*}
\nabla_{a_j}L_n(\bm{\theta})&=2\sum_{i=1}^n\ell'(-y_i f(x_i;\bm{\theta}))(-y_i)(a_j\bm{u}_j^\top (x_i\odot \phi_j))_++2\lambda_ja_j
\end{align*}
since it is straightforward to see that the loss function $L_n$ is differentiable with respect to the vector $\bm{a}$.

\begin{lemma}\label{lemma::gradient-descent}
In the use of the gradient descent algorithm, if we choose $\eta_k=\frac{1}{72\max[L_n(\bm{\theta}(k, 0)), 2n]}$, then all of the following three statements are true:
\begin{itemize}
    \item[(1)] for any $k=1,2,3,...$ and for any $t=0,1,...,T-1$, we have
    $$\sgn(\alpha_j(t))=\sgn(\alpha_j( 0))\quad\text{and}\quad|\alpha_j(t+1)- a_j(t)|\le |\alpha_j( t)|/2,$$
    \item[(2)] for any $k=0,1,2,...$, we have
    $$L_n({\bm{\Theta}}(T))\le L_n(\bm{\Theta}(0)) - \sum_{t=0}^{T-1}\frac{\|\nabla_{\bm{\alpha}}L_n(\bm{\Theta}( t))\|_2^2}{144\max[L_n(\bm{\theta}(0)), 2n]},$$
    \item[(3)] for any $k=1,2,...$, we have 
    $$\min_{t\in[T]}\|\nabla_{\bm{a}}L_n(\bm{\theta}(t))\|_2^2\le \frac{144L_n(\bm{\theta}(k,0))\max[L_n(\bm{\theta}(0)), 2n]}{T}$$
\end{itemize}

\end{lemma}

\begin{proof}
\textbf{(1)} Recall that $\bm{\theta}(k,t)$ denotes the parameter vector of the neural network at $t$-th step of the gradient descent in the $k$-th outer iteration. By the definition of the gradient, we have 
\begin{align*}
\nabla_{a_j}L_n(\bm{\theta}(k, t))&=-2\sum_{i=1}^n\ell'_i(k, t)y_i(a_j(k, t)\bm{u}_j(k)^\top (x_i\odot \phi_j))_++2\lambda_j(k)a_j(k, t)\\
&=2\left[-\sum_{i=1}^n\ell'_i(k, t)y_i\sgn(a(k, t))(\sgn(a(k, t))\bm{u}_j(k)^\top (x_i\odot \phi_j))_++\lambda_j(k)\right]a_j(k, t)
\end{align*}
where we define $\ell_i'(k,t)=\ell'(-y_i f(x_i;\bm{\theta}(k, t)))$ for $i=1,...,n$ and the second equation follows from the fact that $(ab)_+=(\sgn(a)|a|b)_+=|a|(\sgn(a)b)_+=\sgn(a)(\sgn(a)b)_+a$ holds for any $a, b\in\mathbb{R}$.
Therefore, this indicates that 
\begin{align*}
a_j(k, t+1)&=a_j(k, t)-\eta\nabla_{a_j}L_n(\bm{\theta}(k, t))\\
&=a_j(k, t)-2\eta\left[-\sum_{i=1}^n\ell'_i(k, t)y_i\sgn(a(k, t))(\sgn(a(k, t))\bm{u}_j(k)^\top (x_i\odot \phi_j))_++\lambda_j(k)\right]a_j(k, t)\\
&=\left[1-2\eta\lambda_j(k)+2\eta\sum_{i=1}^n\ell'_i(k, t)y_i\sgn(a(k, t))(\sgn(a(k, t))\bm{u}_j(k)^\top (x_i\odot \phi_j))_+\right]a_j(k, t).\\
\end{align*}
Since $\eta\le 1/(8n)$, 
\begin{align*}
1-2\eta\lambda_j(k)+2\eta\sum_{i=1}^n\ell'_i(k, t)y_i\sgn(a(k, t))(\sgn(a(k, t))\bm{u}_j(k)^\top (x_i\odot \phi_j))_+&\ge 1-2\eta n-2n\eta\\
&=1-4\eta n\ge1/2,
\end{align*}
and 
\begin{align*}
1-2\eta\lambda_j(k)+2\eta\sum_{i=1}^n\ell'_i(k, t)y_i\sgn(a(k, t))(\sgn(a(k, t))\bm{u}_j(k)^\top (x_i\odot \phi_j))_+&\le 1+2\eta n+2n\eta\\
&=1+4\eta n\le3/2,
\end{align*}
then for any given outer iteration index $k$, for any $t=0, 1,...$, we have 
$$\sgn(a_j(k, t))=\sgn(a_j(k, 0))$$
and
$$|a_j(k, t+1)- a_j(k, t)|\le |a_j(k, t)|/2.$$

\textbf{(2)} We prove it by induction. We first prove that  $L_n(\bm{\theta}(k, t))\le L_n(\bm{\theta}(k, 0))$ holds for any $t=1,2,...$. We prove it by induction. It is easy to check that $L_n(\bm{\theta}(k, t))\le L_n(\bm{\theta}(k, 0))$ holds when $t=0$. Now We assume that at time $t$, $L_n(\bm{\theta}(k, t))\le L_n(\bm{\theta}(k, 0))$. Since the gradient is locally Lipschitz and $\lambda_j(k)\ge \lambda_0/2\ge \sqrt{n}/2$ for any $j$ and $k$, then at time $(k, t+1)$, we should have
\begin{align*}
\frac{\sqrt{n}\|\bm{a}(k, t)\|_2^2}{2}\le \sum_{j=1}^m\lambda_j(k)a_j^2(k, t)\le L_n(\bm{\theta}(k, t))\le L_n(\bm{\theta}(k, 0)),
\end{align*}
\begin{align*}
L_n({\bm{\theta}}(k, t+1))-L_n(\bm{\theta}(k, t))&\le \left<\nabla_{\bm{a}}L_n(\bm{\theta}(k, t)),{\bm{a}}(k, t
+1)-{\bm{a}}(k, t) \right>\\
&\quad+ 12\|\bm{a}(k, t+1)-\bm{a}(k, t)\|_2\sqrt{n^2\|\bm{a}(k,t)\|_2^4+n^2+\lambda_0^2} \\
&\le -\eta_k\|\nabla_{\bm{a}}L_n(\bm{\theta}(k, t))\|_2^2 +12\eta_k^2\sqrt{4L^2_n(\bm{\theta}(k,0))+n^2+\lambda_0^2}\|\nabla_{\bm{a}}L_n(\bm{\theta}(k, t))\|_2^2 \\
&\le \{-\eta_k+36\eta_k^2\max[L_n(\bm{\theta}(k, 0)), 2n]\}\|\nabla_{\bm{a}}L_n(\bm{\theta}(k, t))\|_2^2\\
&=-\frac{\|\nabla_{\bm{a}}L_n(\bm{\theta}(k, t))\|_2^2}{144\max[L_n(\bm{\theta}(k, 0)), 2n]}\le 0,
\end{align*}
where we use the fact that $\eta_k \le  1/(72\max[L_n(\bm{\theta}(k, 0)), 2n])$.
Therefore, at time $(k, t+1)$, we have $L_n({\bm{\theta}}(k, t+1))\le L_n(\bm{\theta}(k, t))\le L_n(k, 0)$. Therefore, for a given outer iteration $k$, we have 
$$L_n({\bm{\theta}}(k, T))\le L_n(\bm{\theta}(k, 0)) - \sum_{t=0}^{T-1}\frac{\|\nabla_{\bm{a}}L_n(\bm{\theta}(k, t))\|_2^2}{144\max[L_n(\bm{\theta}(k, 0)), 2n]}.$$
\textbf{(3)} So this indicates that 
\begin{align*}
\min_{t\in[T]}\|\nabla_{\bm{a}}L_n(\bm{\theta}(k, t))\|_2^2\le \frac{144L_n(\bm{\theta}(k,0))\max[L_n(\bm{\theta}(k, 0)), 2n]}{T}. 
\end{align*}

\end{proof}

\section{Finding Approximate Solutions} 
In this section, we will show how to always find a vector $\bm{v}_j^*$ for each $j\in[d-r+1] $satisfying $$\left|\sum_{i=1}^n \ell'(-y_i f(x_i;\bm{\theta}))y_i ({\bm{v}_j^*}^\top (x_i\odot \phi_j))_+\right|\ge \max_{\bm{u}\in\mathbb{B}^d}\left|\sum_{i=1}^n \ell'(-y_i f(x_i;\bm{\theta}))y_i (\bm{u}^\top (x_i\odot \phi_j))_+\right|-{\lambda_0}.$$
An easy way of doing this is exhaustive searching. Let $I_j=\{k\in[d]: \phi_j(k)=1\}$ and 
$$G(\bm{u};\phi_j,\bm{\theta})=\left|\sum_{i=1}^n \ell'(-y_i f(x_i;\bm{\theta}))y_i (\bm{u}^\top (x_i\odot \phi_j))_+\right|.$$ By defintion, we have $|I_j|\le r$. Therefore, for each $j\in[d-r+1]$, by setting $\bm{u}_j(k)=0$ for all $k\notin I_j$ and exhaustive searching the vectors in the set $\mathcal{A}_j$ $$\mathcal{A}_j=\left\{\bm{u}\in\mathbb{B}^d:\bm{u}(k)=0\text{ for }k\notin I_j\text{ and }\bm{u}(k)\in\bigcup_{h=0}^{\lceil \lambda_0\rceil}\left\{\frac{h\lambda_0}{n}\right\}\text{ for } k\in I_j\right\}$$
for each $k\in I_j$, we set the vector $\bm{v}^*_j$ to be vector $\bm{u}$ maximizing the value of $G(\bm{u};\phi_j,\bm{\theta})$ in the set $\mathcal{A}_j$, i.e.,
$$\bm{v}_j^*=\arg\max_{\bm{u}:\bm{u}\in\mathcal{A}_j}G(\bm{u};\phi_j,\bm{\theta}).$$
Since the set $\mathcal{A}_j$ is discrete, then the exhaustive searching has the computation complexity of $(d-r+1)(\lambda_0+1)^r$ and for any series of vectors $(\phi_j\in\{0,1\}^d)_{j\in[d-r+1]}$, the approximation error is 
$$\max_{j\in[d-r+1]}\left|G(\bm{v}^*_j;\alpha,\phi_j)-\max_{\bm{u}\in\mathbb{S}^{d-1}}G(\bm{u};\alpha,\phi_j)\right|\le \frac{\lambda_0}{n}\sum_{i=1}^n\ell'(-y_i f(x_i;\bm{\theta}))\le \lambda_0.$$

\section{Proof of Theorem~\ref{thm::algo}}\label{appendix::proof::thm-algo}

\begin{proof}

We assume that the algorithm terminates when $k=k_0$ for some positive integer $k_0\le \infty$. Now we first prove the following claim to show that algorithm ends after finite iterations.  

\begin{claim}\label{claim::algo::1}
Let $K=\max\{\lceil L_n(\bm{\theta}(0);\bm{\lambda}(0)), 2n\rceil\}$. Then the following statements are true:
\begin{itemize}
    \item[(1)] $L_n(\bm{\theta}(k);\bm{\lambda}(k))\le L_n(\bm{\theta}(k-1);\bm{\lambda}(k-1))-1$ holds for any integer $1\le k< k_0$;
    \item[(2)] $k_0\le K$;
    \item[(3)] $\|\bm{\lambda}(k)-{\lambda_0}\bm{1}_m\|_\infty\le \frac{k\lambda_0}{2K}$ holds for any $0\le k\le k_0$.
\end{itemize}
\end{claim}
\begin{proof} 
In fact, we only need to prove the statement (1) is correct, since the statement (1) implies (2). This can be seen by the fact that $L_n$ is non-negative by definition. Statement (3) is also trivial if part (2) is true, since $\|\bm{\lambda}(k)-\bm{\lambda}(k-1)\|\le \lambda_0/(2K)$ holds for all $k\ge 0$ and $\bm{\lambda}_0=\lambda_0\bm{1}_m$. Therefore, we only need to prove that the statement (1) is true. 

We prove the statement (1) by induction. For the base where $k=0$, it holds trivially. We now assume that $L_n(\bm{\theta}(k);\bm{\lambda}(k))\le L_n(\bm{\theta}(k-1);\bm{\lambda}(k-1))-1$ holds for all $1\le k< k_0-1$.  Recall that in the  $k$-th iteration, we first find $\bm{\lambda}(k)$ such that
\begin{equation}\label{eq::thm-algo::1}
\min_{p_1,...,p_n\in\mathbb{R}}\sum_{j:\phi_j=\phi_s}\left(\lambda_j(k)-\sgn(a_j(k))\sum_{i=1}^np_iy_i(\sgn(a_j(k))\bm{u}_j^\top(k) (x_i\odot \phi_j))_+\right)^2\ge \left(\frac{\lambda_0}{8K}\right)^2
\end{equation}
\begin{equation}\label{eq::thm-algo::2}
\bm{\lambda}(k)\preccurlyeq \bm{\lambda}(k-1)
\end{equation}
and 
\begin{equation}\label{eq::thm-algo::3}
\|\bm{\lambda}(k)-\bm{\lambda}(k-1)\|_\infty\le \frac{\lambda_0}{2K}.
\end{equation}
Furthermore, when the gradient descent algorithm terminates at the $T_k$-step, we have 
$$\|\nabla_{\bm{a}}L_n(\bm{\theta}(k, T_k))\|_2\le \frac{\lambda_0}{16K\sqrt{n}}.$$
By definition of the gradient descent and the  statement (1) in Lemma~\ref{lemma::gradient-descent}, we have that, for each $s\in[d-r+1]$,
\begin{align*}
&\|\nabla_{\bm{a}}L_n(\bm{\theta}(k, T_k))\|^2_2\\
&=\sum_{j=1}^m\left(\lambda_j -\sum_{i=1}^n\ell'(1-y_if(x_i;\bm{\theta}(k, T_k)))y_i\sgn(a_j(k, 0))(\sgn(a_j(k, 0))\bm{u}_j^\top (x_i\odot \phi_j))_+\right)^2a_j^2(k, T_k)\\
&=\sum_{s=1}^{d-r+1} \sum_{j:\phi_j=\phi_s}\left(\lambda_j -\sum_{i=1}^n\ell'(1-y_if(x_i;\bm{\theta}(k, T_k)))y_i\sgn(a_j(k, 0))(\sgn(a_j(k, 0))\bm{u}_j^\top (x_i\odot \phi_s))_+\right)^2a_j^2(k, T_k)\\
&\ge \left(\frac{\lambda_0}{8K}\right)^2\sum_{s=1}^{d-r+1}\min_{j:\phi_j=\phi_s}a_j^2(k, T_k),
\end{align*}
where the second equality follows from the fact that the series $\{\phi_j\}_j$ is a periodic series of period equal to $d-r+1$. 
Therefore, we have 
$$\sum_{s=1}^{d-r+1}\min_{j:\phi_j=\phi_s}a_j^2(k, T_k)\le \left(\frac{8K}{\lambda_0}\right)^2\|\nabla_{\bm{a}}L_n(\bm{\theta}(k, T_k))\|^2_2\le \left(\frac{8K}{\lambda_0}\right)^2\cdot\frac{\lambda_0^2}{256nK^2}\le \frac{1}{4n}.$$
This indicates that for each feature vector $\phi_s$, $s\in[d-r+1]$, then there exists a neuron of some index $j$ and with $\phi_j=\phi_s$ such that $|a_j|\le 1/\sqrt{4n}$. Recall that, in Algorithm, we first find an index $s\in[d-r+1]$ such that $$\left|\sum_{i=1}^n \ell(-y_i f(x_i;\bm{\theta}(k, T_k)))y_i ({\bm{v}_s^*}^\top (x_i\odot \phi_s))_+\right|> 5\lambda_0$$
and next  find an index $j\in[m]$ such that $\phi_j=\phi_s$ and $|a_j|\le \frac{1}{\sqrt{4n}}$. Furthermore, we set 
 \begin{align*}
      a_j(k+1)&= \sgn\left(\sum_{i=1}^n \ell_i'(k, T_k)y_i ({\bm{v}_s^*}^\top (x_i\odot \phi_s))_+\right)\sqrt{\frac{1}{\lambda_0}},\\
   \bm{u}_j(k+1)&=\sgn(a_j(k+1))\bm{v}_s^*,\\
   a_r(k+1)&= a_r(k, T_k), \quad\bm{u}_r(k+1)=\bm{u}_r(k) \quad\text{for any }r\neq s.
 \end{align*}  
Since 
\begin{align*}
    \|\bm{\lambda}(k)-\lambda_0\bm{1}_m\|_\infty\le \sum_{j=1}^k\|\bm{\lambda}(j)-\bm{\lambda}(j-1)\|_\infty\le \sum_{j=1}^k \frac{\lambda_0}{2K}=\frac{k\lambda_0}{2K}\le \frac{\lambda_0}{2},
\end{align*}
where the last inequality follows from our induction assumption that
$$0\le L_n(\bm{\theta}(k);\bm{\lambda}(k))\le L_n(\bm{\theta}(0);\bm{\lambda}(0))-k$$
and the fact that $K=\lceil L_n(\bm{\theta}(0);\bm{\lambda}(0))\rceil$, then by Lemma~\ref{lemma::perturbation}, we have 
 $$L_n({\bm{\theta}}(k+1);\bm{\lambda}({k}))-L_n(\bm{\theta};\bm{\lambda}(k))\le -\frac{1}{\lambda_0}\left|\sum_{i=1}^n \ell'(-y_i f(x_i;\bm{\theta}))y_i ({\bm{v}^*_{j}}^\top (x_i\odot \phi_{j}))_+\right|+4.$$
Since $\bm{v}_j^*$ satisfies 
 $$\left|\sum_{i=1}^n \ell'(-y_i f(x_i;\bm{\theta}))y_i ({\bm{v}_j^*}^\top (x_i\odot \phi_j))_+\right|\ge 5\lambda_0$$
before the algorithm terminates, then  we have 
$$L_n({\bm{\theta}}(k+1);\bm{\lambda}(k))-L_n(\bm{\theta}(k);\bm{\lambda}(k))\le-1.$$ After we update the regularizer coefficient $\bm{\lambda}$ at the beginning of the $(k+1)$-th iteration such that $\bm{\lambda}(k+1)\preccurlyeq \bm{\lambda}(k)$, we have 
$$L_n(\tilde{\bm{\theta}}(k+1);\bm{\lambda}(k+1))\le L_n(\tilde{\bm{\theta}}(k+1);\bm{\lambda}(k)).$$
Thus, we have 
$$L_n(\tilde{\bm{\theta}}(k+1);\bm{\lambda}(k+1))\le L_n(\tilde{\bm{\theta}}(k);\bm{\lambda}(k))-1\le L_n(\bm{\theta}(0);\bm{\lambda}(0))-(k+1).$$

\end{proof}

Next, we present the following claim to show that when the algorithm terminates at the $k_0$-th iteration, then all of the three conditions in Theorem~\ref{thm::other-gen}
 holds. 
\begin{claim}\label{claim::algo::2}
When the algorithm terminates at the $k_0$-th iteration, then all the following three conditions are true:
\begin{itemize}
    \item[(1)] $\|\bm{w}_j(k_0, T_{k_0})\|_2=|a_j(k_0, T_{k_0})|$ for all $j\in[m]$;
    \item[(2)] $\|\nabla_{\bm{a}}L_n(\bm{\theta}(k, T_{k_0});\bm{\lambda}(k_0))\|_2\le \frac{\lambda_0}{16K\sqrt{n}};$
    \item[(3)] for each $s\in[d-r+1]$, the following inequality holds 
$$ \max_{\bm{u}\in\mathbb{B}^d}\left|\sum_{i=1}^n \ell'(-y_i f(x_i;\bm{\theta}(k_0, T_{k_0}))y_i (\bm{u}^\top (x_i\odot \phi_s))_+\right|\le 5\lambda_0+C;$$
   \item[(4)] $\|\bm{a}({k_0}, T_{k_0})\|_2\le 2\sqrt{\frac{K}{\lambda_0}}.$
\end{itemize}
\end{claim}

\begin{proof}
\textbf{(1)} The first condition holds trivially, since when we use the gradient descent algorithm, we always have $\bm{w}_j(k_0, T_{k_0})=a_j(k_0, T_{k_0})\bm{u}_j(k_0)$ for all $j\in[m]$, which further implies that  $\|\bm{w}_j(k_0, T_{k_0})\|_2=\|a_j(k_0, T_{k_0})\bm{u}_j(k_0)\|_2=|a_j(k_0, T_{k_0})|$ for all $j\in[m]$.

\textbf{(2)} When the gradient descent algorithm terminates at the $T_{k_0}$-th step, by the termination criterion, we have 
$$\|\nabla_{\bm{a}}L_n(\bm{\theta}(k, T_{k_0});\bm{\lambda}(k_0))\|_2\le \frac{\lambda_0}{16K\sqrt{n}}.$$

\textbf{(3)} When the algorithm terminates, for each $s\in[d-r+1]$, we have  $$\left|\sum_{i=1}^n \ell'(-y_i f(x_i;\bm{\theta}(k_0, T_{k_0}))y_i ({\bm{v}_s^*}^\top (x_i\odot \phi_s))_+\right|\le 5\lambda_0.$$

Furthermore, since for each $s\in[d-r+1]$, the vector $\bm{v}_s^*$ satisfies that 
$$\left|\sum_{i=1}^n \ell'(-y_i f(x_i;\bm{\theta}(k_0, T_{k_0}))y_i ({\bm{v}_s^*}^\top (x_i\odot \phi_s))_+\right|\ge \max_{\bm{u}\in\mathbb{B}^d}\left|\sum_{i=1}^n \ell'(-y_i f(x_i;\bm{\theta}(k_0, T_{k_0}))y_i (\bm{u}^\top (x_i\odot \phi_s))_+\right|-C,$$
for some positive constant $C\ge \lambda_0$
then when the algorithm terminates, for each $s\in[d-r+1]$, we have 
$$ \max_{\bm{u}\in\mathbb{B}^d}\left|\sum_{i=1}^n \ell'(-y_i f(x_i;\bm{\theta}(k_0, T_{k_0}))y_i (\bm{u}^\top (x_i\odot \phi_s))_+\right|\le 5\lambda_0+C.$$

\textbf{(4)} By the part (1) of Claim~\ref{claim::algo::1}, we know that $$L_n(\bm{\theta}(k_0);\bm{\lambda}(k_0))< L_n(\bm{\theta}(0);\bm{\lambda}(0))\le K. $$
Further, by the part (2) of Lemma~\ref{lemma::gradient-descent}, we have $$L_n(\bm{\theta}(k_0, T_{k_0});\bm{\lambda}_{k_0})\le L_n(\bm{\theta}(k_0, 0);\bm{\lambda}_{k_0})=L_n(\bm{\theta}(k_0);\bm{\lambda}_{k_0})<2K.$$
From the definition of $L_n$ and the result in part (1), it follows that 
$$L_n(\bm{\theta}(k_0, T_{k_0});\bm{\lambda}_{k_0})\ge \sum_{j=1}^{m}\lambda_j(k_0)a^2_j(k_0, T_{k_0})\ge \frac{\lambda_0}{2}\|\bm{a}({k_0}, T_{k_0})\|_2^2,$$
where the last inequality follows from the part (2) and (3) in Claim~\ref{claim::algo::1}.  
Therefore, we have 
$$\|\bm{a}({k_0}, T_{k_0})\|_2\le 2\sqrt{\frac{K}{\lambda_0}}.$$

\end{proof}

Since the output parameters $\bm{\theta}^*=\bm{\theta}(k_0, T_{k_0})$, then by Theorem~\ref{thm::other-gen}, we have 
\begin{align*}
\mathbb{P}(Yf(X;&\bm{\theta}^*)<0)\\
&=\mathbb{P}(Yf(X;\bm{\theta}(k_0, T_{k_0}))<0)\\
&\le \frac{5\lambda_0+C+2E}{\gamma n\ell'(0)}+\left(\frac{(30\lambda_0+6C+12E)\ln n}{a\gamma \lambda_0} +\frac{6}{\lambda_0}\cdot\frac{\lambda_0}{16K\sqrt{n}}\cdot\sqrt{\frac{4K}{\lambda_0}}\right)\frac{4\sqrt{2}}{\ell'(0)\sqrt{n}}\\
&\quad +\frac{1}{\ell'(0)}\sqrt{\frac{\log(1/\delta)}{2n}}\\
&= \frac{5\lambda_0+C+2E}{\gamma n\ell'(0)}+\left(\frac{(30\lambda_0+6C+12E)\ln n}{a\gamma \lambda_0} +\frac{3}{4\sqrt{nK\lambda_0}}\right)\frac{4\sqrt{2}}{\ell'(0)\sqrt{n}}+\frac{1}{\ell'(0)}\sqrt{\frac{\log(1/\delta)}{2n}}\\
&\le\frac{5\lambda_0+C+2E}{\gamma n\ell'(0)}+\left(\frac{(30\lambda_0+6C+12E)\ln n}{a\gamma \lambda_0} +1\right)\frac{4\sqrt{2}}{\ell'(0)\sqrt{n}}+\frac{1}{\ell'(0)}\sqrt{\frac{\log(1/\delta)}{2n}}\\
&=\mathcal{O}\left(\frac{\lambda_0+C+E}{\gamma n} +\frac{(C+E)\ln n}{\gamma \lambda_0\sqrt{n}}+\sqrt{\frac{\log(1/\delta)}{n}}\right)
\end{align*}
where the last inequality follows from the fact that $n\ge 1, K\ge 1, \lambda_0\ge 1$. When $\lambda_0=\sqrt{n}\ln n$, then the upper bound becomes 
$$\mathbb{P}(Yf(X;\bm{\theta}^*)<0)=\mathcal{O}\left(\frac{C+E}{\gamma n}+\frac{\ln n}{\gamma \sqrt{n}}+\sqrt{\frac{\log(1/\delta)}{n}}\right).$$

\end{proof}

\section{Proof of Corollary~\ref{cor::1}}\label{appendix::proof::corollary-1}
Now we present the proof of Corollary~\ref{cor::1}.
\begin{proof}
The upper holds based on the fact that $C=\lambda_0$ when we are using the exhaustive searching algorithm to solve the optimization problem $\max_{\bm{u}\in\mathbb{B}^d}G(\bm{u})$. Now we computes the computational complexity. By Claim~\ref{claim::algo::1}, we know that the algorithm terminates within $K=\max\{\lceil L_n(\bm{\theta}(0);\bm{\lambda}(0)), 2n\rceil\}$ iterations. In each iteration, we need to choose the coefficient $\bm{\lambda}(k)$, running the gradient descent and perturbing the inactive neuron. 

\textbf{Complexity of choosing $\bm{\lambda}$}. By analysis in Appendix~\ref{appendix::choose-ceof}, we know that the complexity is $\mathcal{O}(dm^2)$. 

\textbf{Complexity of gradient descent. } By the part (3) of Lemma~\ref{lemma::gradient-descent}, we have  $$\min_{t\in[T]}\|\nabla_{\bm{a}}L_n(\bm{\theta}(k, t))\|_2^2\le \frac{144L_n(\bm{\theta}(k,0))\max[L_n(\bm{\theta}(k, 0)), 2n]}{T}\le \frac{144K^2}{T},$$
where the second inequality follows from the part (1) of claim~\ref{claim::algo::1}.
Furthermore, since the gradient descent algorithm terminates at $T_k$-th which is also the first time when the condition  $$\|\nabla_{\bm{a}}L_n(\bm{\theta}(k, T_{k});\bm{\lambda}(k))\|_2\le \frac{\lambda_0}{16K\sqrt{n}}$$ holds. This means that for each $k$, $T_k=\mathcal{O}(nK^4/\lambda_0)$.

\textbf{Complexity of perturbation.} Since we are using the exhaustive searching and it has a complexity of $\mathcal{O}(dn^{r/2})$. 

Therefore, above all the complexity of Algorithm~\ref{algo::main} is $\mathcal{O}\left(dKm^2+dKn^{r/2}+nK^5/\lambda_0\right)$.

\end{proof}

\section{Proof of Lemma~\ref{lemma::algo-performance}}
\begin{proof}
Assume that random vectors $\bm{\omega}_1,...,\bm{\omega}_{2M}$ are i.i.d.  random vectors uniformly distributed on the sphere $\mathbb{S}^{d-1}$ and $a_1=...=a_M=1$, $a_{M+1}=...=a_{2M}=-1$. 
Then for any $r>0$, any $\bm{\beta}\in[0,1]$ and any vectors $\{\bm{v}_j\}_{j=1}^{2M}\subset\mathbb{R}^d$,  we have
\begin{align*}
&\sum_{i=1}^ny_i\beta_i\left[\frac{1}{2M}\sum_{j=1}^{2M}a_j \left[(\bm{\omega}_j+r \bm{v}_j)^\top x_i\right]_+\right]=\sum_{i=1}^ny_i\beta_i\left[\frac{1}{2M}\sum_{j=1}^{2M}a_j \left[\bm{\omega}_j^\top x_i+r \bm{v}_j^\top x_i\right]_+\right]\\
&=\sum_{i=1}^ny_i\beta_i\left[\frac{1}{2M}\sum_{j:|\bm{\omega}^\top_jx_i|\ge r}a_j \left[(\bm{\omega}_j^\top x_i)_++r \bm{v}_j^\top x_i\mathbb{I}\{\bm{\omega}_j^\top x_i\ge 0\}\right]\right]\\
&\quad+\sum_{i=1}^ny_i\beta_i\left[\frac{1}{2M}\sum_{j:|\bm{\omega}^\top_jx_i|< r}a_j \left[\bm{\omega}_j^\top x_i+r \bm{v}_j^\top x_i\right]_+\right]\\
&=\sum_{i=1}^ny_i\beta_i\left[\frac{1}{2M}\sum_{j=1}^{2M}a_j \left[(\bm{\omega}_j^\top x_i)_++r \bm{v}_j^\top x_i\mathbb{I}\{\bm{\omega}_j^\top x_i\ge 0\}\right]\right]\\
&\quad+\sum_{i=1}^ny_i\beta_i\left[\frac{1}{2M}\sum_{j:|\bm{\omega}^\top_jx_i|< r}a_j \left[\bm{\omega}_j^\top x_i+r \bm{v}_j^\top x_i\right]_+\right]\\
&\quad-\sum_{i=1}^ny_i\beta_i\left[\frac{1}{2M}\sum_{j:|\bm{\omega}^\top_jx_i|< r}a_j \left[(\bm{\omega}_j^\top x_i)_++r \bm{v}_j^\top x_i\mathbb{I}\{\bm{\omega}_j^\top x_i\ge 0\}\right]\right]\\
&\ge \sum_{i=1}^ny_i\beta_i\left[\frac{1}{2M}\sum_{j=1}^{2M}a_j \left[(\bm{\omega}_j^\top x_i)_++r \bm{v}_j^\top x_i\mathbb{I}\{\bm{\omega}_j^\top x_i\ge 0\}\right]\right]-\sum_{i=1}^n\beta_i\left[\frac{2r}{M}\sum_{j:|\bm{\omega}^\top_jx_i|< r}1\right]\\
&=I_1 + I_2 +I_3
\end{align*}
where 
\begin{align*}
I_1 &= \sum_{i=1}^{n}y_i \beta_i \left[\frac{1}{2M}\sum_{j=1}^{2M}a_j (\bm{\omega}_j^\top x_i)_+\right]\\
I_2 &= \sum_{i=1}^{n}y_i \beta_i\left[\frac{r}{2M}\sum_{j=1}^{2M}a_j  \bm{v}_j^\top x_i\mathbb{I}\{\bm{\omega}_j^\top x_i\ge 0\}\right]\\
I_3 &= -\sum_{i=1}^n\beta_i\left[\frac{2r}{M}\sum_{j:|\bm{\omega}^\top_jx_i|< r}1\right]
\end{align*}
Now we are going to use the following concentration inequalities to bound the $I_1, I_2$ and $I_3$, respectively. 

\begin{claim}\label{claim::concen::1}
For any $t>0$, we have 
\begin{align*}
\mathbb{P}\left(\forall \bm{\beta}\in[0,1]^n:\left|\sum_{i=1}^{n}y_i \beta_i \left[\frac{1}{2M}\sum_{j=1}^{2M}a_j (\bm{\omega}_j^\top x_i)_+\right]\right|\le t\sum_{i=1}^n\beta_i\right)
\ge 1-2n\exp\left(-4Mt^2\right).
\end{align*}
\end{claim}

\begin{claim}\label{claim::concen-2}
Let $\e_0=\max\{\e, n^{-1/3}\}$. If  $n\ge {\frac{\ln(2/\delta)}{2\e_0^2}}$ and $M\ge\frac{ \ln(4n/\delta)}{\gamma^2}$, we have
\begin{align*}
\mathbb{P}\left(\forall\bm{\beta}\in[0,1]^n:\sum_{i=1}^{n}y_i \beta_i\left[\frac{r}{2M}\sum_{j=1}^{2M}  \bar{\bm{v}}^\top(\bm{\omega}_j) x_i\mathbb{I}\{\bm{\omega}_j^\top x_i\ge 0\}\right]\ge r(\gamma-4\e_0)\sum_{i=1}^n\beta_i \right)\ge 1-\delta.
\end{align*}
\end{claim}

\begin{claim}\label{claim::concen-3}
For any $t>0$, we have 
\begin{align*}
&\mathbb{P}\left(\forall \bm{\beta}\in[0,1]^n:\sum_{i=1}^n\beta_i\left[\frac{2r}{M}\sum_{j:|\bm{\omega}^\top_jx_i|< r}1\right]\le 4r\left(rd+t\right)\sum_{i=1}^n\beta_i\right)\ge 1-2n e^{-4Mt^2}
\end{align*}
\end{claim}

Based on these three claims, when
$$r<\frac{\gamma-4\e_0}{16d},\quad t=\frac{r(\gamma-4\e_0)}{4},\quad n\ge \frac{\ln (6/\delta)}{2\e_0^2}\quad\text{and}\quad M\ge \max\left\{\frac{\ln(4n/\delta)}{\gamma^2},\frac{4\ln(6n/\delta)}{r^2(\gamma-4\e_0)^2}\right\}$$
we know that with probability at least $1-\delta$, for any $\bm{\beta}\in[0,1]^n$, we have 
\begin{align*}
&\sum_{i=1}^{n}y_i \beta_i \left[\frac{1}{2M}\sum_{j=1}^{2M}a_j (\bm{\omega}_j^\top x_i)_+\right] +\sum_{i=1}^{n}y_i \beta_i\left[\frac{r}{2M}\sum_{j=1}^{2M}  \bar{\bm{v}}^\top(\bm{\omega}_j) x_i\mathbb{I}\{\bm{\omega}_j^\top x_i\ge 0\}\right]-\sum_{i=1}^n\beta_i\left[\frac{2r}{M}\sum_{j:|\bm{\omega}^\top_jx_i|< r}1\right]\\
&\ge \left[-t\sum_{i=1}^n\beta_i\right]+\left[r(\gamma-4\e_0)\sum_{i=1}^n\beta_i\right]+\left[ -4r\left(rd+t\right)\sum_{i=1}^n\beta_i\right]\\
&=\left[-t+r(\gamma-4\e_0)-4r(rd+t)\right]\sum_{i=1}^n\beta_i\\
&\ge \frac{r(\gamma-4\e_0)}{8}\sum_{i=1}^n\beta_i.
\end{align*}
Since $\|\bar{\bm{v}}(\bm{\omega})\|_2\le 1$ holds for all $\bm{\omega}\in\mathbb{S}^{d-1}$ and $|a_j|=1$ for all $j\in[2M]$, then we have for all $\bm{\beta}\in[0,1]^n$

\begin{align*}
& \max_{\bm{v}_1,...,\bm{v}_{2M}\in\mathbb{B}^d}\sum_{i=1}^{n}y_i \beta_i\left[\frac{r}{2M}\sum_{j=1}^{2M}  a_j\bm{v}_j^\top x_i\mathbb{I}\{\bm{\omega}_j^\top x_i\ge 0\}\right]\\
 &= \max_{\bm{v}_1,...,\bm{v}_{2M}\in\mathbb{B}^d}\sum_{i=1}^{n}y_i \beta_i\left[\frac{r}{2M}\sum_{j=1}^{2M}  \bm{v}_j^\top x_i\mathbb{I}\{\bm{\omega}_j^\top x_i\ge 0\}\right]\ge\sum_{i=1}^{n}y_i \beta_i\left[\frac{r}{2M}\sum_{j=1}^{2M}  \bar{\bm{v}}^\top(\bm{\omega}_j) x_i\mathbb{I}\{\bm{\omega}_j^\top x_i\ge 0\}\right].
\end{align*}
Furthermore, since 
\begin{align*}
\max_{\bm{v}_1,...,\bm{v}_{2M}\in\mathbb{B}^d}\sum_{i=1}^{n}y_i \beta_i\left[\frac{r}{2M}\sum_{j=1}^{2M}  \bm{v}_j^\top x_i\mathbb{I}\{\bm{\omega}_j^\top x_i\ge 0\}\right]&=\max_{\bm{v}_1,...,\bm{v}_{2M}\in\mathbb{B}^d}\frac{r}{2M}\sum_{j=1}^{2M}\bm{v}_j^\top\left[\sum_{i=1}^n y_i\beta_i x_i \mathbb{I}\{\bm{\omega}_j^\top x_i\ge 0\}\right]\\
&=\frac{r}{2M}\sum_{j=1}^{2M}\max_{\bm{v}_j\in\mathbb{B}^d}\bm{v}_j^\top\left[\sum_{i=1}^n y_i\beta_i x_i \mathbb{I}\{\bm{\omega}_j^\top x_i\ge 0\}\right]\\
&=\frac{r}{2M}\sum_{j=1}^{2M}\left\|\sum_{i=1}^n y_i\beta_i x_i \mathbb{I}\{\bm{\omega}_j^\top x_i\ge 0\}\right\|_2
\end{align*}
where the equality holds when we set 
$$\bm{v}^*_j=\frac{\sum_{i=1}^n y_i\beta_i x_i \mathbb{I}\{\bm{\omega}_j^\top x_i\ge 0\}}{\left\|\sum_{i=1}^n y_i\beta_i x_i \mathbb{I}\{\bm{\omega}_j^\top x_i\ge 0\}\right\|_2}.$$
\end{proof}

This means that
\begin{align*}
&\sum_{i=1}^{n}y_i \beta_i \left[\frac{1}{2M}\sum_{j=1}^{2M}a_j (\bm{\omega}_j^\top x_i)_+\right] +\sum_{i=1}^{n}y_i \beta_i\left[\frac{r}{2M}\sum_{j=1}^{2M}  a_j(a_j\bm{v}_j^*)^\top x_i\mathbb{I}\{\bm{\omega}_j^\top x_i\ge 0\}\right]-\sum_{i=1}^n\beta_i\left[\frac{2r}{M}\sum_{j:|\bm{\omega}^\top_jx_i|< r}1\right]\\
&\ge\sum_{i=1}^{n}y_i \beta_i \left[\frac{1}{2M}\sum_{j=1}^{2M}a_j (\bm{\omega}_j^\top x_i)_+\right] +\sum_{i=1}^{n}y_i \beta_i\left[\frac{r}{2M}\sum_{j=1}^{2M}  \bar{\bm{v}}^\top(\bm{\omega}_j) x_i\mathbb{I}\{\bm{\omega}_j^\top x_i\ge 0\}\right]-\sum_{i=1}^n\beta_i\left[\frac{2r}{M}\sum_{j:|\bm{\omega}^\top_jx_i|< r}1\right]\\
&\ge \frac{r(\gamma-4\e_0)}{8}\sum_{i=1}^n\beta_i.
\end{align*}

This further indicates that 
\begin{align*}
&\sum_{i=1}^ny_i\beta_i\left[\frac{1}{2M}\sum_{j=1}^{2M}a_j \left[(\bm{\omega}_j+ra_j \bm{v}^*_j)^\top x_i\right]_+\right]\\
&\ge \sum_{i=1}^{n}y_i \beta_i \left[\frac{1}{2M}\sum_{j=1}^{2M}a_j (\bm{\omega}_j^\top x_i)_+\right] +\sum_{i=1}^{n}y_i \beta_i\left[\frac{r}{2M}\sum_{j=1}^{2M}  a_j(a_j\bm{v}_j^*)^\top x_i\mathbb{I}\{\bm{\omega}_j^\top x_i\ge 0\}\right]\\
&\quad-\sum_{i=1}^n\beta_i\left[\frac{2r}{M}\sum_{j:|\bm{\omega}^\top_jx_i|< r}1\right]\\
&\ge \frac{r(\gamma-4\e_0)}{8}\sum_{i=1}^n\beta_i.
\end{align*}
Since 
\begin{align*}
    \sum_{i=1}^ny_i\beta_i\left[\frac{1}{2M}\sum_{j=1}^{2M}a_j \left[(\bm{\omega}_j+ra_j \bm{v}^*_j)^\top x_i\right]_+\right]&=\frac{1}{2M}\sum_{j=1}^{2M}a_j\left[\sum_{i=1}^n y_i\beta_i\left[(\bm{\omega}_j+ra_j \bm{v}^*_j)^\top x_i\right]_+\right]\\
    &\le \frac{1}{2M}\sum_{j=1}^{2M}\left|\sum_{i=1}^n y_i\beta_i\left[(\bm{\omega}_j+ra_j \bm{v}^*_j)^\top x_i\right]_+\right|\\
    &\le \max_{j\in[2M]}\left|\sum_{i=1}^n y_i\beta_i\left[(\bm{\omega}_j+ra_j \bm{v}^*_j)^\top x_i\right]_+\right|,
\end{align*}
we have 
\begin{align*}
    \max_{j\in[2M]}\left|\sum_{i=1}^n y_i\beta_i\left[(\bm{\omega}_j+ra_j \bm{v}^*_j)^\top x_i\right]_+\right|\ge \frac{r(\gamma-4\e_0)}{8}\sum_{i=1}^n\beta_i\ge \frac{r(\gamma-4\e_0)}{8}\max_{\bm{u}\in\mathbb{B}^d}\left|\sum_{i=1}^ny_i\beta_i(\bm{u}^\top x_i)\right|.
\end{align*}

\subsection{Proof of Claim~\ref{claim::concen::1}}

\begin{proof}
For  any $t>0$, by Chernoff-Hoeffding's inequality we have 
\begin{align*}
\mathbb{P}\left(\left|\frac{1}{2M}\sum_{j=1}^{2M}a_j (\bm{\omega}_j^\top x_i)_+\right|\ge t\right)\le 2e^{-4Mt^2},
\end{align*}
by the fact that for each $x_i$,
$$\mathbb{E}\left[\frac{1}{2M}\sum_{j=1}^{2M}a_j (\bm{\omega}_j^\top x_i)_+\right]=\frac{1}{2M}\sum_{j=1}^{2M}a_j\mathbb{E}_{\bm{\omega}\sim\mathcal{N}(\bm{0}_d,I_d)}[(\bm{\omega}^\top x_i)_+]=0\cdot \mathbb{E}_{\bm{\omega}\sim\mathcal{N}(\bm{0}_d,I_d)}[(\bm{\omega}^\top x_i)_+]=0.
$$
This indicates that 
\begin{align*}
\mathbb{P}\left(\bigcup_{i=1}^n\left\{\left|\frac{1}{2M}\sum_{j=1}^{2M}a_j (\bm{\omega}_j^\top x_i)_+\right|\ge t\right\}\right)\le \sum_{i=1}^n\mathbb{P}\left(\left|\frac{1}{2M}\sum_{j=1}^{2M}a_j (\bm{\omega}_j^\top x_i)_+\right|\ge t\right)\le 2ne^{-4Mt^2}.
\end{align*}
Therefore, we have

\begin{align*}
\mathbb{P}\left[\left|\sum_{i=1}^{n}y_i \beta_i \left[\frac{1}{2M}\sum_{j=1}^{2M}a_j (\bm{\omega}_j^\top x_i)_+\right]\right|\le t\sum_{i=1}^n\beta_i\right]&\ge\mathbb{P}\left(\bigcap_{i=1}^n\left\{\left|\frac{1}{2M}\sum_{j=1}^{2M}a_j (\bm{\omega}_j^\top x_i)_+\right|\le t\right\}\right)\\
\ge 1-2n\exp\left(-4Mt^2\right).
\end{align*}
\end{proof}

\subsection{Proof of Claim~\ref{claim::concen-2}}

\begin{proof}
By Assumption that there exists a function $\bar{\bm{v}}(\omega):\mathbb{R}^d\rightarrow \mathbb{R}^d$ with $\|\bar{\bm{v}}(\omega)\|_2\le 1$  such that 
$$\mathbb{P}\left(Y\int\bar{\bm{v}}^\top(\bm{\omega})X\mathbb{I}\{X^\top \omega\ge 0\}\mu(d\omega)\ge \gamma\right)\ge 1-\varepsilon,$$
now we will show how to use this assumption to lower bound the probability 
$$\mathbb{P}\left(\frac{1}{n}\sum_{i=1}^n \mathbb{I}\left\{\frac{y_i}{2M}\sum_{j=1}^{2M} \bar{\bm{v}}^\top(\bm{\omega}_j) x_i \mathbb{I}[x_i^\top \bm{\omega}_j\ge 0]\ge\frac{\gamma}{2}\right\}\ge 1-2\varepsilon\right).$$
Now we define the following event, 
\begin{align*}
A&=\left\{\frac{1}{n}\sum_{i=1}^n\mathbb{I}\left\{y_i \int \bar{\bm{v}}^\top(\omega) x_i \mathbb[x_i^\top \omega\ge 0]\mu (d\omega)\ge {\gamma}\right\}\ge 1-2\e_0\right\},\\
B_i&=\left\{\left|\frac{y_i}{2M}\sum_{j=1}^{2M} \bar{\bm{v}}^\top(\bm{\omega}_j) x_i \mathbb{I}[x_i^\top \bm{\omega}_j\ge 0]-y_i \int \bar{\bm{v}}^\top(\omega) x_i \mathbb[x_i^\top \omega\ge 0]\mu (d\omega)\right|\le \frac{\gamma}{2}\right\},\quad i\in[n].
\end{align*}
Then it is easy to see that 
$$\mathbb{P}\left(\frac{1}{n}\sum_{i=1}^n \mathbb{I}\left\{\frac{y_i}{2M}\sum_{j=1}^{2M}\bar{\bm{v}}^\top(\bm{\omega}_j) x_i \mathbb{I}[x_i^\top \bm{\omega}_j\ge 0]\ge\frac{\gamma}{2}\right\}\ge 1-2\varepsilon\right)\ge \mathbb{P}\left\{ A\cap \left(\bigcap_{i=1}^nB_i\right)\right\}.$$
By Chernoff-Hoelffding's inequality, we have 
\begin{align*}
&\mathbb{P}\left(\left|\frac{1}{n}\sum_{i=1}^n\mathbb{I}\left\{y_i \int \bar{\bm{v}}^\top(\omega) x_i \mathbb[x_i^\top \omega\ge 0]\mu (d\omega)\ge {\gamma}\right\}-\mathbb{E}\bm{1}\left[Y\int\bar{\bm{v}}^\top(\bm{\omega})X\mathbb{I}\{X^\top \omega\ge 0\}\mu(d\omega)\ge \gamma\right]\right|\le \varepsilon\right)\\
&\quad\ge 1-2e^{-2n\e_0^2},
\end{align*}
and this indicates that 
\begin{align*}
&\mathbb{P}\left(\frac{1}{n}\sum_{i=1}^n\mathbb{I}\left\{y_i \int \bar{\bm{v}}^\top(\omega) x_i \mathbb[x_i^\top \omega\ge 0]\mu (d\omega)\ge {\gamma}\right\}\ge\mathbb{E}\bm{1}\left[Y\int\bar{\bm{v}}^\top(\bm{\omega})X\mathbb{I}\{X^\top \omega\ge 0\}\mu(d\omega)\ge \gamma\right]- \varepsilon\right)\\
&\quad\ge 1- 2e^{-2n\e_0^2},
\end{align*}
which further implies that 
\begin{align*}
 \mathbb{P}(A)=\mathbb{P}\left(\frac{1}{n}\sum_{i=1}^n\mathbb{I}\left\{y_i \int \bar{\bm{v}}^\top(\omega) x_i \mathbb[x_i^\top \omega\ge 0]\mu (d\omega)\ge {\gamma}\right\}\ge1- 2\varepsilon\right)\ge 1-2e^{-2n\e_0^2} 
\end{align*}
by the fact that 
$$\mathbb{E}\bm{1}\left[Y\int\bar{\bm{v}}^\top(\bm{\omega})X\mathbb{I}\{X^\top \omega\ge 0\}\mu(d\omega)\ge \gamma\right]=\mathbb{P}\left(Y\int\bar{\bm{v}}^\top(\bm{\omega})X\mathbb{I}\{X^\top \omega\ge 0\}\mu(d\omega)\ge \gamma\right)\ge 1-\varepsilon.$$
Now we show the lower bound of $\mathbb{P}(B_i)$ for each $i\in[n]$. By Chernoff inequality, for each $i\in[n]$, we have 
\begin{align*}
    \mathbb{P}(B_i)&=\mathbb{P}\left(\left|\frac{y_i}{2M}\sum_{j=1}^{2M} \bar{\bm{v}}^\top(\bm{\omega}_j) x_i \mathbb{I}[x_i^\top \bm{\omega}_j\ge 0]-y_i \int \bar{\bm{v}}^\top(\omega) x_i \mathbb[x_i^\top \omega\ge 0]\mu (d\omega)\right|\le \frac{\gamma}{2}\right)\ge 1-2e^{-M\gamma^2}.
\end{align*}
Above all, by union bounds 
\begin{align*}
\mathbb{P}\left\{ A^c\cup \left(\bigcup_{i=1}^{n} B_i^c\right)\right\}\le \mathbb{P}(A^c)+\sum_{i=1}^n \mathbb{P}(B_i^c)\le 2e^{-2n\e_0^2}+2ne^{-M\gamma^2}.
\end{align*}
This means that if $n\ge {\frac{\ln(2/\delta)}{2\e_0^2}}$ and $M\ge\frac{ \ln(4n/\delta)}{\gamma^2}$, then 
\begin{align*}
\mathbb{P}\left\{ A^c\cup \left(\bigcup_{i=1}^{n} B_i^c\right)\right\}\le \mathbb{P}(A^c)+\sum_{i=1}^n \mathbb{P}(B_i^c)\le {\delta}.
\end{align*}
This indicates that 
$$\mathbb{P}\left(\frac{1}{n}\sum_{i=1}^n \mathbb{I}\left\{\frac{y_i}{2M}\sum_{j=1}^{2M} \bar{\bm{v}}^\top(\bm{\omega}_j) x_i \mathbb{I}[x_i^\top \bm{\omega}_j\ge 0]\ge\frac{\gamma}{2}\right\}\ge 1-2\varepsilon\right)\ge \mathbb{P}\left\{ A\cap \left(\bigcap_{i=1}^nB_i\right)\right\}\ge 1-\delta.$$
Thus, if $n\ge {\frac{\ln(2/\delta)}{2\e_0^2}}$ and $M\ge\frac{ \ln(4n/\delta)}{\gamma^2}$, then
\begin{align*}
\mathbb{P}\left(\forall\bm{\beta}\in[0,1]^n:\sum_{i=1}^{n}y_i \beta_i\left[\frac{r}{2M}\sum_{j=1}^{2M} \bar{ \bm{v}}^\top(\bm{\omega}_j) x_i\mathbb{I}\{\bm{\omega}_j^\top x_i\ge 0\}\right]\ge r(\gamma-4\e_0)\sum_{i=1}^n\beta_i \right)\ge 1-\delta.
\end{align*}

\end{proof}

\subsection{Proof of Claim~\ref{claim::concen-3}}
\begin{proof}
By Chernoff-Hoeffding's inequality, for each $i\in[n]$, we have 
\begin{align*}
\mathbb{P}\left(\left|\frac{1}{2M}\sum_{j:|\bm{\omega}^\top_jx_i|< r}1-\mathbb{P}\left(|\bm{\omega}^\top x_i|<r\right)\right|\ge t \right)\le 2e^{-4Mt^2}.
\end{align*}
Therefore, by union bounds, we have 
\begin{align*}
\mathbb{P}\left(\bigcup_{i=1}^n\left\{\frac{1}{2M}\sum_{j:|\bm{\omega}^\top_jx_i|< r}1\ge \mathbb{P}\left(|\bm{\omega}^\top x_i|<r\right)+t\right\}\right)\le 2n e^{-4Mt^2},
\end{align*}
which further indicates that 
$$\mathbb{P}\left(\bigcap_{i=1}^n\left\{\frac{1}{2M}\sum_{j:|\bm{\omega}^\top_jx_i|< r}1\le \mathbb{P}\left(|\bm{\omega}^\top x_i|<r\right)+t\right\}\right)\ge 1-2n e^{-4Mt^2}.$$
Therefore, 
\begin{align*}
&\mathbb{P}\left(\forall \bm{\beta}\in[0,1]^n: \sum_{i=1}^n\beta_i\left[\frac{2r}{M}\sum_{j:|\bm{\omega}^\top_jx_i|< r}1\right]\le 4r\sum_{i=1}^{n}\beta_i\left(\mathbb{P}\left(|\bm{\omega}^\top x_i|<r\right)+t\right)\right)\\
&\ge\mathbb{P}\left(\bigcap_{i=1}^n\left\{\frac{1}{2M}\sum_{j:|\bm{\omega}^\top_jx_i|< r}1\le \mathbb{P}\left(|\bm{\omega}^\top x_i|<r\right)+t\right\}\right)\ge 1-2n e^{-4Mt^2}
\end{align*}

Now we only need to calculate the following probability,
\begin{align*}
\mathbb{P}\left(|\bm{\omega}^\top x_i|<r\right),
\end{align*}
where the random vector $\bm{\omega}$ is a uniform random vector on the surface of a $d$-dimensional ball, i.e., $\bm{\omega}\sim U(\mathbb{S}^{d-1})$. We know that if $\bm{\omega}_{\mathcal{N}}$ is a $d$-dimensional Gaussian random vector of mean $\bm{0}_d$  and variance matrix $I_d$, then the random vector  $\bm{\omega}_{\mathcal{N}}/\|\bm{\omega}_{\mathcal{N}}\|_2\sim U(\mathbb{S}^{d-1})$. Therefore, for each $i\in[n]$, we have 
\begin{align*}
\mathbb{P}\left(|\bm{\omega}^\top x_i|<r\right)&=\mathbb{P}_{\bm{\omega}_{\mathcal{N}}\sim \mathcal{N}(\bm{0}_d, I_d)}\left(|\bm{\omega}_{\mathcal{N}}^\top x_i|<r\|\bm{\omega}_{\mathcal{N}}\|_2\right)\\
&=\int_{0}^{+\infty}\mathbb{P}\left(|\bm{\omega}_{\mathcal{N}}^\top x_i|<r\|\bm{\omega}_{\mathcal{N}}\|_2, \|\bm{\omega}_{\mathcal{N}}\|^2_2=t^2\right)f_{\chi^2(d)}(t^2)dt&&\text{by }\|\bm{\omega}_{\mathcal{N}}\|_2^2\sim\chi^2(d)\\
&\le \int_{0}^{+\infty}\mathbb{P}\left(|\bm{\omega}_{\mathcal{N}}^\top x_i|<rt^2\right)f_{\chi^2(d)}(t^2)dt\\
&\le \int_{0}^{+\infty}\frac{rt^2}{\sqrt{2\pi}}f_{\chi^2(d)}(t^2)dt\\
&=\frac{r}{\sqrt{2\pi}}\int_{0}^{+\infty}t^2f_{\chi^2(d)}(t^2)dt\\
&\le \frac{r}{\sqrt{2\pi}}\left(\int_{0}^{+\infty}tf_{\chi^2(d)}(t^2)dt+\int_{0}^{+\infty}t^3f_{\chi^2(d)}(t^2)dt\right)\\
&= \frac{r}{2\sqrt{2\pi}}\left(\int_{0}^{+\infty}f_{\chi^2(d)}(t^2)d(t^2)+\int_{0}^{+\infty}t^2f_{\chi^2(d)}(t^2)d(t^2)\right)\\
&=\frac{r}{2\sqrt{2\pi}}\left(\int_{0}^{+\infty}f_{\chi^2(d)}(t)dt+\int_{0}^{+\infty}tf_{\chi^2(d)}(t)dt\right)\\
&=\frac{r}{2\sqrt{2\pi}}\left(1+d\right)\le rd.
\end{align*}
This indicates that 
\begin{align*}
&\mathbb{P}\left(\forall\bm{\beta}\in[0,1]^n:\sum_{i=1}^n\beta_i\left[\frac{2r}{M}\sum_{j:|\bm{\omega}^\top_jx_i|< r}1\right]\le 4r\left(rd+t\right)\sum_{i=1}^n\beta_i\right)\\
&\ge\mathbb{P}\left(\forall\bm{\beta}\in[0,1]^n:\sum_{i=1}^n\beta_i\left[\frac{2r}{M}\sum_{j:|\bm{\omega}^\top_jx_i|< r}1\right]\le 4r\sum_{i=1}^{n}\beta_i\left(\mathbb{P}\left(|\bm{\omega}^\top x_i|<r\right)+t\right)\right)\\
&\ge 1-2n e^{-4Mt^2}
\end{align*}
\end{proof}

\end{appendix}

\end{document}